\documentclass[twoside,11pt]{article}

\usepackage{jmlr2e}
\usepackage{listings} 
\usepackage{xcolor}
\usepackage{amsfonts, mathtools, hyperref, subcaption}
\usepackage{booktabs} % for professional tables
\usepackage{natbib}
\usepackage[bottom]{footmisc}
\usepackage{tikz}   
\usepackage{lastpage} 
\usetikzlibrary{svg.path}  
\usetikzlibrary{cd} 
\mathtoolsset{showonlyrefs}
\newtheorem{property}{Property}

\lstdefinestyle{vscode-dark}{
    backgroundcolor=\color{black}, % Background color
    basicstyle=\ttfamily\footnotesize\color{white}, % Monospaced font, smaller size, white text
    keywordstyle=\bfseries\color{cyan}, % Keywords in bold and cyan
    commentstyle=\itshape\color{gray}, % Comments in italic and gray
    stringstyle=\color{yellow}, % Strings in yellow
    numberstyle=\tiny\color{gray}, % Line numbers in tiny font and gray
    stepnumber=1, % Line numbers at every step
    numbersep=10pt, % Space between line numbers and code
    showstringspaces=false, % Don't show spaces in strings
    breaklines=true, % Automatic line breaking 
    frame=single, % Frame around the code
    framerule=0.5pt, % Thickness of the frame
    captionpos=b, % Caption position at the bottom
}

\newcommand{\rd}{\mathrm{d}}

\newcommand{\N}{\mathcal{N}} 
\newcommand{\R}{\mathbb{R}}

\ShortHeadings{Symplectic neural networks}{BK Tapley}
\firstpageno{1}

\begin{document}

\title{Symplectic Neural Networks Based on Dynamical Systems}

\author{\name Benjamin Kwanen Tapley \email bentapley@hotmail.com \\
       \addr Department of Mathematics and Cybernetics\\ 
       SINTEF Digital\\
       0373 Oslo, Norway
      }

\editor{} 
 
\maketitle  

\begin{abstract}%   
  We present and analyze a framework for designing symplectic neural networks (SympNets) based on geometric integrators for Hamiltonian differential equations. The SympNets are universal approximators in the space of Hamiltonian diffeomorphisms, interpretable and have a non-vanishing gradient property. We also give a representation theory for linear systems, meaning the proposed P-SympNets can exactly parameterize any symplectic map corresponding to quadratic Hamiltonians. Extensive numerical tests demonstrate increased expressiveness and accuracy --- often several orders of magnitude better --- for lower training cost over existing architectures. Lastly, we show how to perform symbolic Hamiltonian regression with SympNets for polynomial systems using backward error analysis.
\end{abstract}

\begin{keywords} 
	Structure-Preserving Machine Learning; Physics-Informed Machine Learning; Symplectic Neural Networks; Hamiltonian Neural Networks; Dynamical Systems; Scientific Machine Learning. 
\end{keywords}
  
\section{Introduction}
Structure-preserving neural networks have recently garnered interest among the machine-learning community due to their effectiveness in learning solutions to complex problems from sparse data sets. A neural model is called ``structure-preserving" when it incorporates inductive biases, such as physics or geometry, into either the model architecture or training. This can be achieved by either penalizing non-structure-preserving model candidates in the loss function like when training a physics-informed neural network \citep{raissi2019physics} or by incorporating such structures directly into the neural architecture. The former approach is useful when the structure one wants to enforce is complex, such as a partial differential equation but results in models that are not exactly structure-preserving. The latter approach is useful when the structure is simpler, such as a symplectic structure, and can result in more parameter efficient models that are applicable to a wide range of problems. This latter approach is the focus of this paper.

In particular, we address the problem of learning a symplectic mapping from time-series data $\{x(t_i), x(t_i+h)\}_{i=1}^{n_{\text{data}}}$, where $\phi_{h}(x(t_i))=x(t_i+h)$ is an unknown map that generates the data. We assume that $\phi_{h}$ preserves a canonical symplectic structure denoted on the conjugate phase space coordinates $x(t)=(p(t),q(t))\in\R^{2n}$ by the closed, non-degenerate, canonical two-form 
$$\omega=\sum_{i=1}^{n} \rd p^i\wedge \rd q^i.$$ 
Such a map is called a canonical symplectic transformation, which by definition leaves the symplectic matrix $J=$ {\tiny ${ \left(\begin{array}{cc}
	O & -I \\
	I & O \\
	\end{array}\right)}$ } $\in\R^{2n\times 2n}$
invariant
\begin{equation}
(\frac{\partial\phi_h(x)}{\partial x})^TJ(\frac{\partial\phi_h(x)}{\partial x})=J,
\end{equation} 
where $\frac{\partial\phi_h(x)}{\partial x}$ is the Jacobian matrix of the flow map at the point $x$ and $I$, $O$ are the $n\times n$ identity and zero matrices, respectively. Canonical symplectic transformations often arise from the solution of a Hamiltonian ordinary differential equation (ODE) of the form 
\begin{equation}\label{hode}
\dot{p} = - \frac{\partial H}{\partial q},
\quad \dot{q} = \frac{\partial H}{\partial p},
\end{equation}

for Hamiltonian function $H:\R^{2n}\rightarrow\R$. If this is the case, then we call $\phi^H_h:=\phi_h$ a Hamiltonian flow or map, which is a subset of symplectic maps. 

The primary goal of this study is to develop a framework for parameterizing arbitrary Hamiltonian maps $\phi^{H}_h$ by neural networks that are symplectic by construction. The motivation for learning symplectic dynamics is threefold. First, Hamiltonian systems arise naturally in many physical scenarios, such as chemical reactions, accelerator physics, electrodynamics or ideal mechanical systems, the dynamics of which can be measured and information about the system can be learned. Second, symplectic maps can be used as a building block for learning dynamics of more realistic physical systems. Symplectic maps (or integrators) are one of the cornerstones of geometric numerical methods \citep{hairer2006geometric} and are used to solve a wide range of realistic problems in the physical sciences and engineering through their use in composition and splitting methods (e.g., \citep{tapley2019novel,tapley2022computational}). There are many examples where structure-preserving neural networks, including symplectic neural networks (SympNets), have been applied, including systems with Poisson structure \citep{jin2022learning}, systems with forces such as dissipation \citep{eidnes2023pseudo}, volume-preserving systems \citep{bajars2023locally} or optimal control \citep{meng2022sympocnet}. Third, Hamiltonian flows and dynamical system-based neural networks are beginning to play a useful role in more general machine machine-learning applications such as normalizing flows \citep{chen2018neural}, equivariant flows \citep{rezende2019equivariant}, deep neural networks with non-vanishing gradients \citep{galimberti2023hamiltonian}, generative flows \citep{toth2019hamiltonian} or classification problems \citep{haber2017stable}, to name a few. In fact, it has been recently shown in \cite{zakwan2023universal} that Hamiltonian deep neural networks have a universal approximation property for arbitrary maps, which further motivates their useful role outside purely physical systems.

The core idea behind our proposed framework is simply to approximate the unknown Hamiltonian as a sum $H\approx \sum_{i=1}^kH_i(x)$, then construct a symplectic splitting method by composing the exact Hamiltonian flows on each $H_i(x)$. By leveraging many known results from geometric numerical integration, we can prove a number of interesting properties. These properties as well as our main contributions are briefly summarized:
\begin{itemize}
	\item We develop a framework for SympNets and propose several novel SympNet architectures for parameterizing a Hamiltonian flow map that we refer to as P-SympNets, R-SympNets and GR-SympNets.
	\item We give universal approximation results for these architectures and others within the proposed framework (Theorem \ref{thm: universality}).
	\item We give a representation result for linear Hamiltonian systems for P-SympNets (Theorem \ref{thm: representation of linear hamiltonian flows}). 
	\item We demonstrate that our SympNets are interpretable and are amenable to symbolic regression algorithms using backward error analysis to identify the true Hamiltonian (Section \ref{sec: symbolic regression}).
	\item We develop a Python package that implements the SympNets and symbolic regression algorithms that can be installed via \texttt{pip install strupnet}. 
\end{itemize}

\subsection{Previous Work and Motivation}\label{sec: previous work and motivation}
The theoretical foundations of approximating symplectic maps stem from a result due to \cite{turaev2002polynomial}, who show that iterations of H\'enon-like maps of the form $(p,q)\mapsto (q+\eta, -p+\nabla V(q))$ are dense in the space symplectomorphisms in the $C^\infty$-topology. This result was used in \cite{jin2020sympnets} to prove symplectic universality of SympNet architectures of the following form\footnote{Throughout the paper we will always use $\phi^{H}_h$ to denote the \textit{exact} flow of a Hamiltonian ODE $\dot{x}=J\nabla H$. Furthermore, $\Phi^H_h$ will always be used to denote a \textit{composition} of Hamiltonian flows.}
\begin{equation}\label{SympNet fixed directions}
	\Phi^{H^\theta}_h=\phi^{T^\theta_k(p)}\circ\phi^{V^\theta_k(q)}\circ\cdots\circ\phi^{T^\theta_1(p)}\circ\phi^{V^\theta_1(q)}
\end{equation}
where $T^\theta_i(p)$ and $V^\theta_i(q)$ are trainable functions (e.g., neural networks) depending on either $p$ or $q$ only. The functions $\phi^{T^\theta_i(p)}$ and $\phi^{V^\theta_i(q)}$ can be thought of as symplectic residual neural network layers as they take the form $x+F(x)$, where $F$ is a symplectic transformation. Since the development of this novel architecture, several other authors propose learning symplectic dynamics with SympNets of the form \eqref{SympNet fixed directions} \citep[e.g.,][]{chen2019symplectic, tong2021symplectic, valperga2022learning, burby2020fast, maslovskaya2024symplectic, horn4555181generalized}.

As Hamiltonian systems are so ubiquitous, several more efficient algorithms for learning symplectic dynamics have since been proposed including \citep{chen2021data, david2023symplectic, xiong2020nonseparable, offen2022symplectic,chen2023variational}, however these methods learn the dynamics through maps that are defined implicitly, and hence require solving a set of non-linear equations to infer new dynamics. Although many of these methods yield excellent algorithms for learning symplectic structure, a primary motivation of this study to develop SympNets that can be used as building blocks for more complex neural architectures. Therefore, implicit methods will yield intractable algorithms. For this reason we will focus the discussion only on explicit maps of which there are plenty already available in the aforementioned literature.

The common conclusion throughout these studies is that using neural models that are intrinsically symplectic is generally advantageous when the data is also generated by a symplectic processes (e.g., a Hamiltonian ODE). This is largely due to the fact that learning a flow map $\phi_h:\R^{2n}\rightarrow\R^{2n}$, in general, involves learning all $2n$ components of $\phi_h$, whereas assuming canonical symplectic structure in the data uniquely defines the flow map by a \textit{single} scalar function $H:\R^{2n}\rightarrow\R$ that can be learned instead, which significantly reduces the size of the model's hypothesis space. 
  
A disadvantage of the SympNets of the form \eqref{SympNet fixed directions} is that they are less interpretable when the true Hamiltonian is generated by a non-separable Hamiltonian. This is illustrated by the following toy example. We train two SympNets, each with four layers, on a data set generated by a $2n=2$ dimensional Hamiltonian ODE \eqref{hode} with $H(p, q) = \frac{1}{2}  (p^2 + 2pq + 3q^2)$. The first SympNet is a G-SympNet \citep{jin2020sympnets} which is of the form \eqref{SympNet fixed directions} where $T^\theta_i(p)$ and $V^\theta_i(q)$ are neural networks with one hidden layer and certain activation functions. The second SympNet, proposed in this paper, is called a P-SympNet. Over 50000 training epochs, the G-SympNet reaches an MSE training set loss of about 1e-7. The P-SympNet reaches a training set loss of 1e-16 after about 400 epochs. 

\begin{figure}
	\centering 
	\begin{subfigure}{\linewidth}
		\centering  
		\includegraphics[width=0.32\linewidth]{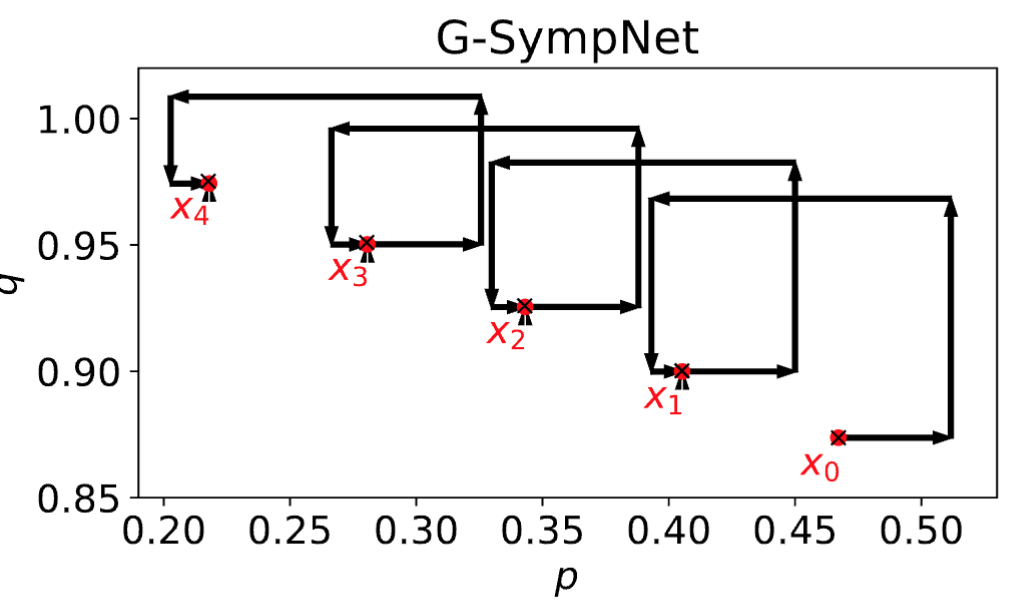}
		\includegraphics[width=0.32\linewidth]{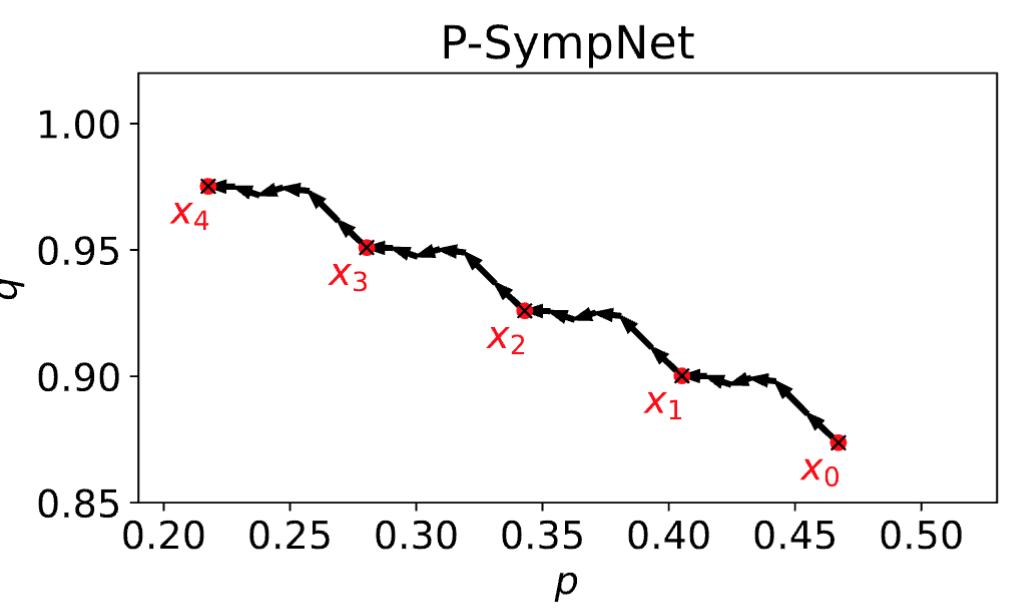}
		\caption{A visualization of how the input $x$ is propagated through the six layers of the SympNets. Each arrow represents a layer and the SympNet implements the map $x_i\mapsto x_{i+1}$. Four iterations of the map is shown, and the black crosses are the exact solution.}\label{fig: traj}
	\end{subfigure}
	\begin{subfigure}{\linewidth}
		\centering
		\includegraphics[width=0.32\linewidth]{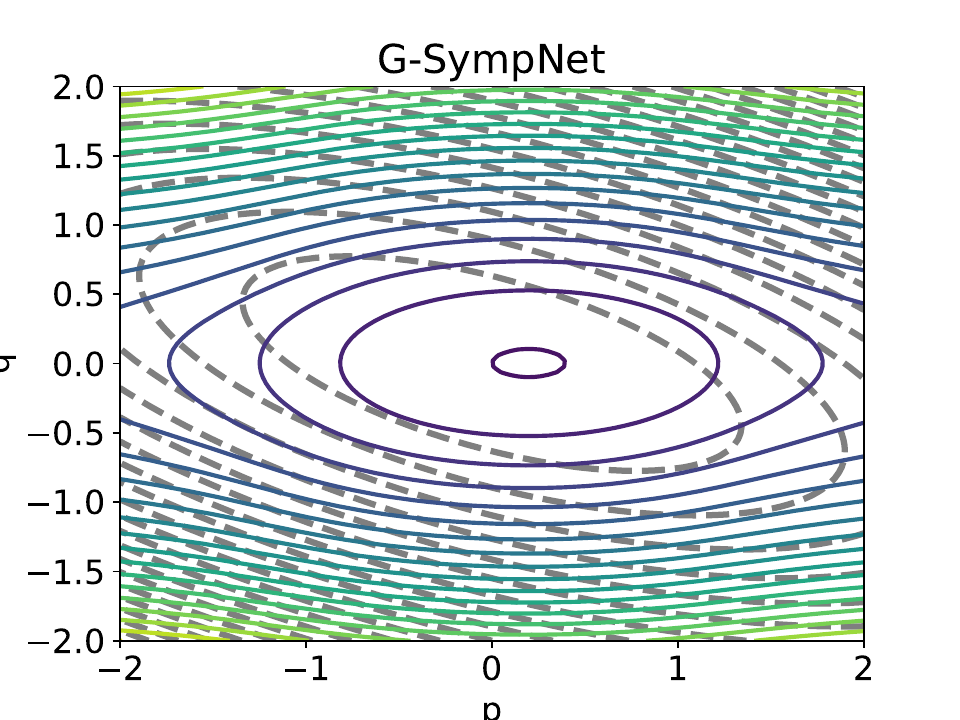}
		\includegraphics[width=0.32\linewidth]{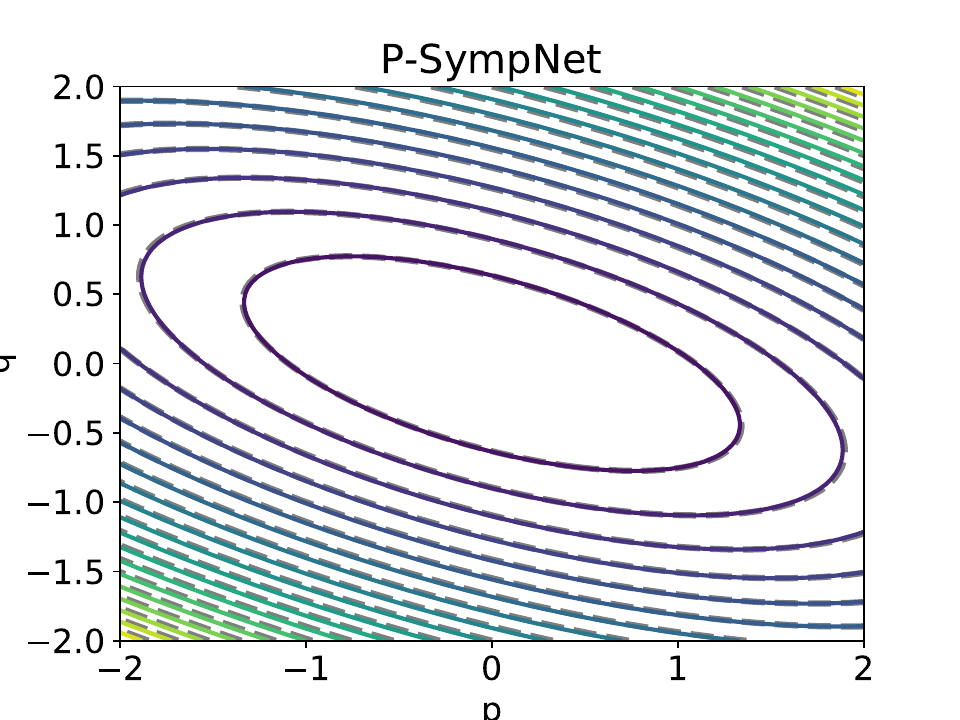}	
		\caption{The level sets of the learned \textit{inverse modified Hamiltonians}. The dashed line is the true solution.}\label{fig: hams}
	\end{subfigure}
	\caption{}\label{fig: comparison}
\end{figure} 

While the G-SympNet yields good predictions for the dynamics, it learns large values for $|T_i^\theta(x)|$ and $|V_i^\theta(x)|$ relative to the size of the true Hamiltonian, resulting in layers $\phi^{V^\theta_i(q)}$ and $\phi^{T^\theta_i(p)}$ that map their inputs in large and opposite directions. This can be seen by figure \ref{fig: traj} that depicts how a data point $x$ is propagated through the layers of the G-SympNet. This point is important because it means the usual tools available to use from numerical analysis do not converge and is evident in figure \ref{fig: hams} by the fact that the learned (inverse modified) Hamiltonian is not close to the true Hamiltonian. The P-SympNet, on the other hand, learns a series of symplectic layers that can be analyzed as though it were a splitting method. This means we can sum up the learned Hamiltonians of the symplectic layers to get an $O(h)$ approximation to the true Hamiltonian. This is called the inverse modified Hamiltonian and will be central to our discussion. Furthermore, we can recover the true Hamiltonian from the inverse modified Hamiltonian by expanding it in its Baker-Cambpell-Hausdorff (BCH) series, which is not possible for the G-SympNet. 

In fact, the G-SympNet \textit{necessarily} learns Hamiltonian layers that are large, so that the BCH formula doesn't converge because if it did, it would learn an inverse modified Hamiltonian close to a separable one, which doesn't agree with the true solution. Due to this, the parameters of the neural networks are biased \textit{away} from zero, which makes them difficult to initialize and harder to train. For these and similar problems, the P-SympNet often learns layers that are close to the identity transformation, and hence initializing them with small random numbers is an appropriate inductive bias. This is evident by small arrows in figure \ref{fig: traj}, representing near-identiy transformations for each layer of the P-SympNet. 

This motivates the present paper, which aims to extend and unify many of these SympNets into a common framework using tools from geometric numerical integration. Within this framework, we can also propose several new SympNets of which we present three that we refer to as P-SympNets (polynomial), R-SympNets (ridge) and GR-SympNets (generalized ridge). These SympNets are of the form,
\begin{equation} \label{eq: my sympnets}
	\Phi^{H^{\theta}}_h=\phi^{{H}^{\theta}_k(x)}_h\circ\dots\circ\phi^{{H}^{\theta}_1(x)}_h,
\end{equation}
where the main difference from \eqref{SympNet fixed directions} is that the layers $\phi^{{H}^{\theta}_i(x)}_h$ use information from both $p$ and $q$ \textit{simultaneously}, allowing for increased expressiveness and interpretability. Instead of relying on the symplectic universality results of \citep{turaev2002polynomial} to prove universality of the SympNets, the above framework allows us to leverage the vast body of literature from dynamical systems theory, namely geometric numerical integration \citep{hairer2006geometric}, to analyze their properties. 

In the remainder of the paper we will recall some theory from backward error analysis and geometric numerical integration of Hamiltonian vector fields. We then outline a framework for constructing SympNets, and then suggest a few models that fit within it. We then present numerical experiments including Hamiltonian system identification and give concluding remarks in the final section.

\section{Parameterizing a Hamiltonian Map} \label{sec: background}
The SympNets proposed in this paper exploit many ideas from the field of geometric numerical integration. This field is concerned with creating structure-preserving maps that possess geometric properties, for example, preservation of first integrals \citep{tapley2022geometric}, second integrals \citep{tapley2023preservation}, volume forms \citep{kang1995volume}, measures \citep{celledoni2019using} and so forth. In recent years, many studies have shown that dynamical systems theory and numerical methods are intricately linked to neural networks, see e.g., \citep{haber2017stable, chang2019antisymmetricrnn, sherry2024designing,celledoni2023dynamical,chen2018neural,galimberti2023hamiltonian}. Hence, it is natural to use similar tools to develop and analyze structure-preserving neural networks as we will do here. We refer to \citep{hairer2006geometric} for a comprehensive introduction to geometric numerical integration. In particular, we will assume some familiarity splitting methods, see e.g., \citep{mclachlan2002splitting,mclachlan2004explicit}.

The concept of the (inverse) backward error analysis is central to the problem of learning dynamics and is presented Section \ref{sec:bea}. Here, we begin with recalling some known theory to do with the classical backward error analysis for splitting methods. We then outline the inverse case, which is relevant for our situation. Next, we leverage this theory and discuss how a composition of symplectic maps can approximate an arbitrary Hamiltonian flow. This leads to our universal approximation result and is a key component to our SympNet framework. Lastly, we show how one can define the layers in equation \eqref{eq: my sympnets} using solutions of Hamiltonian shear vector fields to construct efficient SympNet architectures within this framework. 

\subsection{Backward Error Analysis}\label{sec:bea}
Here we will discuss some concepts that are central to parameterizing an unknown map $\phi^{H}_h$, namely backward error analysis (BEA), which is a tool to analyze the behavior of numerical methods. Letting $X_H=J\nabla H$ denote the Hamiltonian vector field corresponding to Hamiltonian $H:\R^{2n}\rightarrow\R$, denote its time-$h$ flow by $\phi^{H}_h$ and $\Phi^{H}_h\approx\phi^{H}_h$ a symplectic splitting method for $X_H$. Then BEA computes a \textit{modified} ODE $\dot{x}=X_{\tilde{H}}(x)=X_{H}(x)+O(h)$ whose exact flow is the numerical method applied to $\dot{x}=X_H(x)$. That is, $\Phi_h^{H}=\phi_h^{\tilde{H}}$ (see the right-hand side of figure \ref{fig:BEA}). 

For example, take the $k=2$ map $\Phi^H_h=\phi^{H_1}_h\circ\phi^{H_2}_h$ for the splitting $H=H_1(x)+H_2(x)$. If the Hamiltonian is separable: $H_1(x)=T(p)$, $H_2(x)=V(q)$, then this corresponds to the symplectic Euler method. It is shown using BEA that the symplectic Euler method is the exact flow of an ODE with a modified Hamiltonian $\tilde{H}$ \citep{hairer2006geometric}[Chapt. IX.4]. Further, $\tilde{H}$ can be approximated by its Baker-Cambpel-Hausdorff (BCH) series, whose first few terms read
\begin{align}\label{invmod bch}
\tilde{H}  =  H_1 + H_2 + \frac{h}{2}\{H_1, H_2\} + \frac{h^2}{12}\big(\{H_1, \{H_1, H_2\}\} + \{H_2, \{H_2, H_1\}\}\big) +O(h^3), 
\end{align}
where $\{H_1,H_2\}=\nabla H_1^T J \nabla H_2$ is the Poisson bracket. We note that despite the fact that $\tilde{H}$ always exists, the BCH series does not always converge and therefore should be truncated appropriately. Consider now the $k>1$ map $\Phi_h^{H} = \phi^{H_1}_h\circ\dots\circ\phi^{{H}_k}_h$ for the Hamiltonian splitting $H = \sum_{i=1}^k H_i$. Then by induction, one can show that this is the exact flow of a system with the following modified Hamiltonian
\begin{equation}\label{mod ham}
\tilde{H} = \sum_{i=1}^k H_i + \frac{h}{2}\sum_{i<j\le k}\{H_i, H_j\}+O(h^2)
\end{equation}
where the second summation is over the $\frac{1}{2}k(k-1)$ pairs of indices $(i,j)$ where $i<j$ and $j\le k$.  Note that the $O(h^2)$ and higher terms are computable, and involve nested Poisson brackets of the basis Hamiltonians $H_i$.
\begin{figure}
	\centering
	\begin{tikzcd}[column sep={8em,between origins}]	
		X_{\bar{H}} \arrow[dotted]{r}{\mathrm{BEA}} \arrow[dashed]{rd}[anchor=center,rotate=-25,yshift=1ex]{\text{numerical}} & X_{H} \arrow[dotted]{r}{\mathrm{BEA}} \arrow{d}{\text{exact}} \arrow[dashed]{rd}[anchor=center,rotate=-25,yshift=1ex]{\text{numerical}} & X_{\tilde{H}}  \arrow{d}{\text{exact}} \\
		& \Phi_h^{\bar{H}}=\phi_h^{H}\qquad                               & \Phi_h^{H}=\phi_h^{\tilde{H}}\qquad
	\end{tikzcd}
	\caption{A diagram outlining the relationship between the original Hamiltonian ODE $X_{H}=J\nabla H(x)$, the modified ODE $X_{\tilde{H}}$, the inverse modified ODE $X_{\bar{H}}$ and their numerical methods $\Phi_h$ and exact flows $\phi_h$. The dashed lines denote application of the numerical method, the solid lines denote the exact flow and the dotted lines denote application of backward error analysis (BEA).}\label{fig:BEA}
\end{figure}
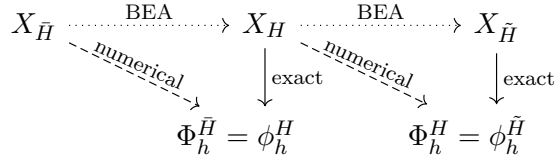
%That is, the method \eqref{composition} is the exact flow of a Hamiltonian system whose Hamiltonian is $O(h)$ away from the true Hamiltonian. For example, the symplectic Euler method is the exact flow of a non-separable system despite the fact that the basis matrices $A_i$ and $B_i$ are chosen in a way that exploits the separability of the original Hamiltonian $H$. 
\subsubsection{Inverse Modified Hamiltonians}
The inverse BEA scenario differs in that we do not know the true Hamiltonian ODE. Instead we only have observations (i.e., data) of its flow map $\phi_h^{H}$. We would like to parameterize it with a $k$-layer SympNet of the form 
\begin{equation}\label{lie trotter map}
\Phi_h^{\bar{H}} = \phi^{\bar{H}_k}_h\circ\dots\circ\phi^{\bar{H}_1}_h
\end{equation}
such that, ideally, the following holds
\begin{equation}\label{true flow}
\Phi_h^{\bar{H}} = \phi_h^{H}.
\end{equation} 
We define the inverse Hamiltonian for a map of the form \eqref{lie trotter map} as the following. 
\begin{definition}[Inverse modified Hamiltonian of $\Phi_h^{\bar{H}}$]
	The inverse modified Hamiltonian of the map \eqref{lie trotter map} is the function $	\bar{H} = \sum_{i=1}^{k}\bar{H}_i.$
\end{definition}
The inverse BEA is depicted on the left-hand side of figure \ref{fig:BEA} and was studied in \cite{zhu2020deep}. The inverse modified Hamiltonian can be found by making the substitution $H\rightarrow\bar{H}$ and $\tilde{H}\rightarrow H$ in \eqref{mod ham}. That is,
\begin{equation}\label{true ham from inv mod ham}
H =\sum_{i=1}^{k}\bar{H}_i + \frac{h}{2}\sum_{i<j\le k}\{\bar{H}_i, \bar{H}_j\}+O(h^2). 
\end{equation}
As Hamiltonian flows form a group under composition \citep{polterovich2012geometry} some $H$ in \eqref{true flow} is guaranteed to exist. The question we explore next is how to find a set $\{\bar{H}_i\}_{i=1}^{k}$ such that equation \eqref{true flow} holds to arbitrary precision.

\subsubsection{Learning an Inverse Modified Hamiltonian}
In general, it is not possible to know what form $\bar{H}$ should take given no knowledge of the true Hamiltonian $H\in C^1(\R^{2n})$. It is instead logical to parameterize $\bar{H}$ with trainable basis functions $\bar{H}^\theta = \sum_{i=1}^k \bar{H}^{\theta}_i$, where $\theta\in\Theta$ represent some parameters and $\Theta$ is the vector space that $\theta$ are allowed to vary over. Ideally, we would like to choose the basis functions to form a map $\Phi_h^{\bar{H}^{\theta}} = \phi^{\bar{H}^{\theta}_k}_h\circ\dots\circ\phi^{\bar{H}^{\theta}_1}_h$
so that there exist some $\theta$ such that $\Phi_h^{\bar{H}^{\theta}} \approx \phi^{H}_h
$ to arbitrary precision (i.e., universal approximation). In other words, we would like that $\bar{H}$ be represented by superpositions of the basis functions
\begin{equation}\label{inv mod ham span}
\bar{H} \in \text{span}\{\bar{H}^{\theta}_i|\,\, \theta \in\Theta \},
\end{equation}
or that this space is dense on compact subsets in $C^1(\R^{2n})$ according to the following definition. 
\begin{definition}
	Let $\Omega\subset\R^n$ be compact. We say that a set $\mathcal{H}$ is $m$-dense, or dense in $C^m(\Omega)$, if there exists some $H^{\theta}\in\mathcal{H}$ such that for any $H\in C^m(\Omega)$ and $\epsilon>0$ 
\begin{equation}\label{density}
\|H(x) - H^{\theta}(x)\|_{C^m(\Omega)}<\epsilon,\quad \forall x\in \Omega,
\end{equation}
where the norm $\|\cdot\|_{C^m(\Omega)}$ is defined as 
$$\|f\|_{C^m(\Omega)} = \sup_{|\alpha| \leq m} \sup_{x \in \Omega} \left| D^\alpha f(x) \right|,$$ $\alpha$ is a multi-index and $D^\alpha=(\frac{\partial}{\partial x_1})^{\alpha_1}...(\frac{\partial}{\partial x_n})^{\alpha_n}$. 
\end{definition}
Doing so allows us to find an arbitrarily close approximation to the true inverse modified Hamiltonian. Given such a set, the parameters $\theta$ can be found by optimizing a loss function on the data set $\mathcal{D}=\{x_i, \phi^{H}_h(x_i)\}_{i=1}^{n_{\text{data}}}$ given by 
\begin{equation}\label{loss}
loss = \sum_{i=1}^{n_{\text{data}}} \| \Phi_h^{\bar{H}^{\theta}}(x_i) - \phi^{H}_h(x_i) \|^2.
\end{equation}
If \eqref{inv mod ham span} holds, then in theory the loss function can be minimized to machine-precision as we will see in the quadratic $H$ case, for carefully selected $\bar{H}_i^\theta$. However, in general this is unachievable. In practice, we would like to choose some basis functions that are dense in compact subset of $C^1(\R^{2n})$. Assuming we can approximate the inverse modified Hamiltonian with arbitrary accuracy, then $\|\Phi_h^{\bar{H}^{\theta}}-\phi_h^{{H}}\|$ can in theory be made arbitrarily small. This property is called universal approximation of Hamiltonian flows and is summarized by the following theorem. 
\begin{theorem}[Universal approximation of Hamiltonian flows]\label{thm: universality}
	Let $\Omega\subset \R^{2n}$ denote some compact subset, assume $H^{\theta}:\R^{2n}\rightarrow\R$ are such that  $\mathrm{span}\{H^{\theta}(x)| \theta\in\Theta\}$ is dense in $C^1(\Omega)$ and let $\phi_h^{H}(x)$ denote a time-$h$ Hamiltonian flow corresponding to the vector field $X_H=J\nabla H$, given some $H\in C^1(\Omega)$. Then there exists some set of $k$ functions $\{H^{\theta}_i\}_{i=1}^k$ such that for any $\epsilon>0$
	\begin{equation}
	\|\phi^{H^{\theta}_k}_h\circ ... \circ \phi^{H^{\theta}_1}_h(x)-\phi^{H}_h(x)\| < \epsilon,
	\end{equation}
	for all $x\in \Omega$.
\end{theorem}
The proof is given in Appendix \ref{sec:proof universality}. In \citep{celledoni2023dynamical}, universality results are given for deep neural architectures by framing a neural network as a discrete-time solution to a dynamical system, which is analagous to ours. While Theorem \ref{thm: universality} ensures that a SympNet of the form \eqref{lie trotter map} is dense in Hamiltonian diffeomorphisms, we are still faced with the highly non-trivial task of computing the exact flows $\phi_h^{\bar{H}^{\theta}_i}$ for a dense set of functions $\bar{H}^{\theta}$ comprising the SympNet layers. This is the subject of Section \ref{sec:symplectic neural networks}. In the next section we will make concrete our general framework and outline some properties of such SympNets.

\subsection{A Framework for Symplectic Neural Networks}	

In this section we will present and analyze a framework for constructing SympNets, which are defined, quite generally, as follows. 

% Let $\mathcal{H} \subset C^1(\R^{2n})$ be a space of differentiable functions. 
% \begin{equation}
% 	\mathcal{SN}^k(H^\theta)=\{ \phi_h^{\bar{H}^{\theta}_1}\circ ... \circ \phi_h^{\bar{H}^{\theta}_k}| \frac{\mathrm{d}}{\mathrm{d} t}\phi_t^H = X_H(\phi_t^H),\quad H\in\mathcal{H} \}
% \end{equation}
% denote the set of all $k$ composition 
\begin{definition}[Symplectic neural network]\label{def:SympNet}
	A $k$-layer symplectic neural network is a map
	\begin{equation}\label{composition}
	\Phi^{\bar{H}^\theta}_h(x) = \phi_h^{\bar{H}^{\theta}_k}\circ ... \circ \phi_h^{\bar{H}^{\theta}_1}(x),
	\end{equation}
	where $\phi_h^{\bar{H}^{\theta}_i}$ are exact Hamiltonian flows, called layers, and the set $\mathcal{H} = \{\bar{H}^{\theta}_i|{\theta}_i\in\Theta\}$ is called the basis (or generating) Hamiltonian set. The inverse modified Hamiltonian of $\Phi^{\bar{H}^\theta}_h(x)$ is given by $\bar{H}^\theta = \sum_{i=1}^k\bar{H}^{\theta}_i$. 
\end{definition}

For brevity, we will usually assume that $\mathcal{H}$ is scale invariant, meaning if $ H\in \mathcal{H}$ then so is $\alpha H$ for $\alpha\in\R$. This framework relies on the fact that we can find a set of basis Hamiltonians such that the layers $\phi_h^{\bar{H}_i^\theta}$ can be computed, which is the subject of the next section. For now, we will assume that this is possible. Framing a SympNet is such a way allows us to leverage a number of concepts from dynamical systems theory and numerical analysis. This has been repeatedly demonstrated in several studies to be a powerful tool when analyzing properties of neural networks, e.g., \citep{haber2017stable, chang2019antisymmetricrnn, sherry2024designing,celledoni2023dynamical,chen2018neural,galimberti2023hamiltonian, celledoni2023learning}. We will now outline some properties of such SympNets. The first and most important one is the following.
 
\begin{property}[Universality]
	If $\mathrm{span} (\mathcal{H})$ is dense in $C^1(\Omega)$, then Theorem \ref{thm: universality} guarantees that $\Phi_h^{\bar{H}^{\theta}}$, according to Definition \ref{def:SympNet}, is dense in Hamiltonian diffeomorphisms. 
\end{property}

Another advantage is due to the following obvious lemma.

\begin{lemma} \label{lemma: small basis} \label{property 2 small basis}
	Let $\mathrm{span}(\mathcal{H})$ be dense in $C^1(\Omega)$. Then there exists a set of functions $\{H^{\theta}_i\}_{i=1}^k$ such that for any $H\in C^1(\Omega)$ and $\epsilon>0$  
	\begin{itemize}
		\item[(i)]$ \|H - \sum_{i=1}^k H^{\theta}_i\|_{C^1(\Omega)}<\epsilon$, and 
		\item[(ii)]$\|H^{\theta}_i\|_{C^1(\Omega)} < \epsilon, \quad \forall i=1,...,k$.
	\end{itemize}
\end{lemma}

Statement $(i)$ is a direct statement of the assumption of density, whilst $(ii)$ says that it is possible to bound the norms of each basis function $H^{\theta}_i$, which is advantageous from a numerical analysis point of view and has important consequences if one is interested in system identification and symbolic Hamiltonian regression.

\begin{property}[Interpretability]
	Due to Lemma \ref{lemma: small basis} $(ii)$, there exists a set of functions $\{H^{\theta}_i\}_{i=1}^k$ whose Hamiltonian flows $\phi^{H^{\theta}_i}_h$ represent transformations that are close to the identity and can be expanded into a \textit{converging} power series. 
\end{property}

This is not necessarily true for an arbitrary SympNet, as we have already seen in the toy example of Section \ref{sec: previous work and motivation}. Furthermore, we can then use inverse BEA to identify the true Hamiltonian, which we will discuss in detail in Section \ref{sec: symbolic regression}. 

Next, we observe that SympNets of this form are simply splitting methods of Hamiltonian flows, which form a group under composition \citep{polterovich2012geometry}. That is, if $\phi^{H_1}_h$ and $\phi^{H_2}_h$ are the flows of the Hamiltonian vector fields $X_{H_1}$ and $X_{H_2}$, then $\phi^{H_2}_h\ circ\phi^{H_1}_h$ is a Hamiltonian flow of the Hamiltonian vector field $X_{K}$, where $K=H_2+H_1\circ(\phi_h^{H_2})^{-1}$. 
\begin{property}[Group property]
	SympNets according to Definition \ref{def:SympNet} are Hamiltonian flows and form a group under composition. That is, if $\Phi_h^{H}$ and $ \Phi_h^{K}$ are SympNets, then $\Phi_h^{H}\circ\Phi_h^{K}$ is also a SympNet and hence a Hamiltonian flow. Furthermore, the inverse is the map with negative timestep and their layers reversed
\begin{equation}
\left(\Phi_h^{\bar{H}^\theta}\right)^{-1} = \phi_{-h}^{\bar{H}_1^\theta}\circ ... \circ \phi_{-h}^{\bar{H}_k^\theta},
\end{equation}
\end{property} 
Furthermore, we have the following obvious property, which is true of all symplectic flows. 
\begin{property}[Volume preservation]
	SympNets according to Definition \ref{def:SympNet} preserve volume in phase space, meaning
	\begin{equation} 
		\det(\frac{\partial\phi^{\bar{H}^\theta}_h(x)}{\partial x}) = 1.
	\end{equation}
\end{property} 
This has been suggested to yield more stability in training. 

Another property of SympNets is that they circumvent the vanishing gradient problem. This is a well-known phenomenon in deep learning that places limitations on the depth of neural networks. This arises during a gradient descent step of the form 
\begin{equation}
	\theta_{m+1} = \theta_m - \eta \nabla_{\theta} \mathcal{L}(\theta_m),
\end{equation}
where $\mathcal{L}$ is the loss function. The idea is to find stationary points of the loss function by iteratively updating the parameters $\theta$ in the direction of the negative gradient. Denote by $y_j:=\phi^{\bar{H}_j^\theta}_h\circ\dots\circ\phi^{\bar{H}_1^\theta}_h$ the output of the $j$th layer of a $k$-layer SympNet $\Phi_h^{\bar{H}^\theta}=y_k$. Then we can express the gradient for the $i$th parameter of layer $j$ by 
\begin{equation} 
	\frac{\partial \mathcal{L}}{\partial \theta_{i,j}} = \frac{\partial y_{j+1}}{\partial \theta_{i,j}} \frac{\partial \mathcal{L}}{\partial y_{j+1}} = \frac{\partial y_{j+1}}{\partial \theta_{i,j}} \left( \prod_{l=j+1}^{k-1} \frac{\partial y_{l+1}}{\partial y_{l}} \right) \frac{\partial \mathcal{L}}{\partial y_{k}}.
\end{equation}
The matrix $\frac{\partial y_{k}}{\partial y_{j+1}}=\left( \prod_{l=j+1}^{k-1} \frac{\partial y_{l+1}}{\partial y_{l}} \right)$ which is referred to as the backward stability matrix is the product of the Jacobians of the layers with respect to their inputs. As the number of layers increases and the Jacobians become small then $\frac{\partial y_{k}}{\partial y_{j+1}}$ vanishes and $\frac{\partial \mathcal{L}}{\partial \theta_{i,j}}$ tends to zero despite not reaching a minimum. This is known as the vanishing gradient problem. The opposite case, where $\frac{\partial y_{k}}{\partial y_{j+1}}$ becomes very large, is known as the exploding gradient problem. As shown in \cite[Theorem 3]{galimberti2023hamiltonian}, when the layers are symplectic, one can bound the norms of $\frac{\partial y_{k}}{\partial y_{j+1}}$, for $j=0,...,k-1$, from below. 
\begin{lemma}\citep{galimberti2023hamiltonian}
	Let $\|\cdot\|$ be any sub-multiplicative matrix norm, then
	$$\|\frac{\partial y_{k}}{\partial y_{j+1}}\|\ge 1,$$for $j=0,...,k-1$. 
\end{lemma}
\begin{proof}
	Using the fact that $\frac{\partial y_{k}}{\partial y_{j+1}}$ are symplectic matrices, we have that
	\begin{equation}
		\|J\| = \|\left(\frac{\partial y_{k}}{\partial y_{j+1}}\right)^TJ\left(\frac{\partial y_{k}}{\partial y_{j+1}}\right)\| \le \|\frac{\partial y_{k}}{\partial y_{j+1}}\|^2\|J\|
	\end{equation}
	for $j=0,...,k-1$, which implies the result.
\end{proof} 
We therefore have the following. 
\begin{property}[Non-vanishing gradients]
	SympNets according to Definition \ref{def:SympNet} do not suffer from the vanishing gradient problem.
\end{property}
SympNets could suffer from exploding gradients, however this can also be circumvented by adding a L2 regularization term to the loss function, also explored in \cite{galimberti2023hamiltonian}. Due to Lemma \ref{lemma: small basis} $(ii)$, there exists solutions to the basis Hamiltonians that are small, meaning that an L2 regularization term makes sense and does not introduce an erroneous inductive bias. Furthermore, in Section \ref{sec: learned parameters} we verify that the proposed SympNets usually learn small values for the parameters. 

So, given a set of basis Hamiltonians $\{\bar{H}^{\theta}_i\}_{i=1}^k$ that are dense on compact sets in $C^1(\R^{2n})$ then the composition method $\Phi^{\bar{H}^\theta}_h$ will possess the above properties. However, we still need to compute the flows $\phi_h^{\bar{H}^{\theta}_i}$, also called layers. We will now explain how this can be achieved. 

\subsection{Designing SympNet Layers Using Hamiltonian Shear Maps}\label{sec: shear vector fields theory}
After we have chosen an appropriately expressive set of basis functions $\{\bar{H}_i^\theta\}$, the layers of a SympNet are constructed by composing their Hamiltonian flows $\phi_h^{\bar{H}_i}$. The key idea explored in this section is that if they yield \textit{shear vector fields}, then the layers $\phi_h^{\bar{H}_i}$ can be computed exactly using forward Euler steps (i.e., residual network layers). We will eventually see that many previously explored SympNets use a special (restricted) case of this framework. 

\begin{definition}[Shear vector fields] \label{eq: shear vf}
	We call $f:\R^{2n}\rightarrow\R^{2n}$ a \textit{shear vector field }if, under some linear transformation $y=Tx$, it can be expressed in the following form where  $y=(y_1^T, y_2^T)^T$
	\begin{equation}
	\begin{split}
	\dot{y}_1=0\in\R^{n},\\
	\dot{y}_2=g_2(y_1)\in\R^{n},
	\end{split}
	\end{equation}
	where $g_2:\R^{n}\rightarrow\R^{n}$.
\end{definition}
The flow of a shear vector field is therefore constant along the $y_1$ direction and linear along the $y_2$ direction in analogy to shear flows in fluid dynamics. We will henceforth restrict the discussion to the Hamiltonian case as in \cite{feng1998variations}, which studies shear Hamiltonian vector fields that are nilpotent of degree two, that is $f(f(x))=0$. In other words, the second time derivative of its solution $x(t)$ vanishes:
\begin{equation}
\ddot{x}(t)=\frac{\mathrm{d}}{\mathrm{d}t}J\nabla H(x) = J\nabla^2 H(x) J\nabla H(x) =0.
\end{equation}
Therefore, a Hamiltonian vector field (and its Hamiltonian function) is called nilpotent of degree two if $H$ satisfies $J \,\nabla^2 H \,J\, \nabla H = 0$, where $\nabla^2$ is the Laplacian operator. 

\begin{lemma}\citep{feng1998variations}
	A Hamiltonian function $H:\R^{2n}\rightarrow\R$ is nilpotent of degree two if and only if it can be written in the form 
	\begin{equation}\label{symmetry cond}
	H(x) = K(C x), \text{ where }	CJC^T=0,
	\end{equation}
	for some $K:\R^n\rightarrow\R$ and $C\in\R^{n\times 2n}$.
\end{lemma}
Nilpotent Hamiltonians are also called shear Hamiltonians as in \cite{koch2014hamiltonian}. 
Given such a function, this yields the ODE 
\begin{equation}\label{shear ode}
\dot{x}=JC^T\nabla K(Cx).
\end{equation}
In addition to $K\circ C$, the ODE has an additional $n$ linear invariants $Cx$, which are in involution with the Hamiltonian, thus yielding an integrable system in the sense of Louiville-Arnold. Writing $C=(A,B)$, for $A,B\in\R^{n\times n}$ then equations \eqref{symmetry cond} become 
$$H(p,q) = K(Ap+Bq), \quad \text{where} \quad AB^T=BA^T.$$
Clearly all Hamiltonian vector fields that are nilpotent of degree two are also shear vector fields. This can be seen by letting $T=(C^T, T_2^T)^T\in\R^{2n\times 2n}$ denote a linear transformation, for some $T_2\in\R^{n\times 2n}$ such that $y=Tx$ and $C$ satisfies the condition \eqref{symmetry cond}. Then writing the ODE for $y=(y_1^T, y_2^T)^T$ and inserting the expression for \eqref{shear ode} yields the form of definition \eqref{eq: shear vf}.
%\begin{equation}
%\left(\begin{array}{c}
%\dot{y}_1\\ \dot{y}_2\\
%\end{array}\right) =
%\left(\begin{array}{c}
%C\dot{x}\\ T_2\dot{x}\\
%\end{array}\right) \\
%=
%\left(\begin{array}{c}
%CJC^T\nabla H_i(Cx)\\ 
%T_2JC^T\nabla H_i(Cx)\\
%\end{array}\right) \\ 
%=
%\left(\begin{array}{c}
%0\\ 
%T_2JC^T\nabla H_i(y_1)\\
%\end{array}\right)
%\end{equation} 
%which is of the form given in equation \eqref{eq: shear vf}.

The remarkable property of a shear vector field is that their solutions are linear in time and are therefore integrated \textit{exactly} by the forward Euler method
\begin{equation}\label{fe shear}
\phi^{K\circ C}_h(x) = x + h JC^T\nabla K(Cx),
\end{equation}
where $\nabla K(Cx)\in\R^n$. This is the ResNet-like architecture that we will use to construct the layers of the SympNets. In the next section we will explore different choices for $K$ and $C$ to construct a SympNet.

\section{Symplectic Neural Networks}\label{sec:symplectic neural networks}
Here we will give some concrete examples of some SympNets that fit within our framework. A common thread of the forthcoming methods is that the layers $\phi_h^{\bar{H}^{\theta}_i}(x)$ are given by equation \eqref{fe shear}. However, we remark that actually \textit{any} symplectic map can be used which would result in a SympNet according to Definition \ref{def:SympNet}. The main distinguishing feature between the forthcoming methods is that they use different choices for the basis shear Hamiltonian functions $K$ and the basis matrices $C$ from equation \eqref{fe shear}. 

This framework lends itself naturally to several approaches. The first we outline uses one-dimensional projections using ridge functions in Section \ref{sec: scalar ridge functions}. Section \ref{sec: generalised ridge functions} addresses an approach using generalized ridge functions. In Section \ref{sec: shear hamiltonians depending on p or q} we discuss existing methods from the literature that use generalized ridge functions with \textit{fixed} directions (i.e., in the $p$ or $q$ direction only).

\subsection{Shear Hamiltonians Using Ridge Functions} \label{sec: scalar ridge functions}
In this section we consider an efficient choice for the basis Hamiltonians using ridge functions. Ridge functions project high dimensional input onto lines along a ``direction vector" $w$ via $\alpha(w^Tx)$, where $\alpha : \R\rightarrow\R$ is a univariate function. We refer the reader to \citep{pinkus2015ridge, ismailov2020notes} for a comprehensive overview of ridge functions in the context of neural networks. We note that $\alpha(w^Tx)$ is a shear due to $w^TJw=0$, which is equivalent to setting all elements of $C$ to zero except for one row in equation \eqref{symmetry cond}. 

For a shear vector field of one variable the map \eqref{fe shear} can be written as
\begin{equation}\label{scalar shear layer}
\phi_h^{\alpha\circ w}(x) = x + h \, \alpha'(w^Tx) J w,
\end{equation}
where $\alpha\circ w = \alpha(w^T \cdot )$ and $\alpha'(y) = \frac{\rd }{\rd y} \alpha(y)$. That is, the map is given as the product of a scalar function and a constant vector. This means the maps parameters scale by $O(n)$ as opposed to $O(n^2)$ for fully connected neural networks. We now consider two choices for the ridge function $\alpha$: (1) one-dimensional neural nets; and (2) univariate polynomials.
\subsubsection{R-SympNets}
For the first choice, let $\alpha = \N:\R\rightarrow\R$ be a neural network of width $m$
\begin{equation}\label{one dim neural net}
\N(y) = \sum_{i=0}^m a_i\sigma(y + b_i)
\end{equation}
where $a,b\in\R^m$ are trainable parameters, $\sigma$ is a non-polynomial activation function and $y$ is some scalar input. Then the shear Hamiltonian is given by $\N(w^Tx)$. 
\begin{definition}[R-SympNets]
	An R-SympNet of width $m$ is a map $ \Phi_h^{\bar{H}^\theta}(x)$ according to Definition \ref{def:SympNet} where the basis Hamiltonian set is given by $$\mathcal{H} = \{\sum_{i=1}^ma_i\sigma(w^Tx + b_i)|\,(a,b,w)\in\R^{m}\times\R^{m}\times\R^{n}\}.$$
\end{definition}
That is, their layers are defined by equation \eqref{scalar shear layer} with $\alpha(w^Tx) = \N(w^Tx)$. It is well known that neural networks of this type can approximate smooth functions and their derivatives to arbitrary precision.   
\begin{lemma}\citep{hornik1991approximation}\label{thm:nonpoly ridge density}    
	Let $\sigma\in C^m(\R)$ be a bounded,  non-constant activation function. Then the set
	\begin{equation} 
	\mathrm{span}\{\sigma(w^Tx + b):\, b\in\R,\,w\in\R^{n}\}
	\end{equation}
	is $m$-dense on compact subsets in $C^m(\R^{n})$ for any $m,n\in\mathbb{N}_+$. 
\end{lemma}
This result means we can apply Theorem \ref{thm: universality} to the R-SympNet. 
\begin{corollary}
There exists an R-SympNet $\Phi_h^{\bar{H}^\theta}$ such that for any $H\in C^1(\R^{2n})$ and any $\epsilon>0$ 
\begin{equation}
\|\Phi_h^{\bar{H}^\theta}(x)-\phi_h^{H}(x)\|<\epsilon
\end{equation}
for $x$ in some compact $\Omega \subset\R^{2n}$. 	
\end{corollary}
\begin{proof}
	Using Lemma \ref{thm:nonpoly ridge density} we can approximate any $H\in C^1(\R^{2n})$ to arbitrary precision by a linear combinations of functions of the form \eqref{one dim neural net}. Then the result follows from Theorem \ref{thm: universality}.
\end{proof}

\subsubsection{P-SympNets}
The next choice we consider are when $\alpha$ are univariate polynomials. This is advantageous for at least two reasons. One, if the true Hamiltonian is smooth or polynomial, then it's inverse modified Hamiltonian is well approximated by a polynomial; and two, polynomials are very fast to evaluate and back propagate through. We denote by $p(z) = \sum_{i=0}^d a_i z^i$ a degree $d$ univariate polynomial. When the scalar input is the one-dimensional projection $z=w^Tx$, then $p(w^Tx)$ is also known as polynomial ridge function \citep{shin1995ridge} and are used in \cite{mclachlan2004explicit} to construct similar geometric numerical methods for polynomial ODEs. 
\begin{definition}[P-SympNets]\label{def:psympnet}
	A P-SympNet of degree $d$ is a map $ \Phi_h^{\bar{H}^\theta}(x)$ according to Definition \ref{def:SympNet} where the basis Hamiltonian set given by $$\mathcal{H} = \{\sum_{i=1}^da_i(w^Tx)^i|\,(a,w)\in\R^{d}\times\R^{{2n}}\}.$$
\end{definition}
That is, the layers are given by \eqref{scalar shear layer} with $\alpha(w^Tx) = p(w^Tx)$. In practice, we will let the sum run from 2 to $d$ to avoid linear terms in the Hamiltonian, an additional physical assumption that would otherwise correspond to a constant term in the ODE. The inverse modified Hamiltonian of a P-SympNet is therefore a polynomial. Note that we can represent any polynomial of degree $d$ in $n$ variables by a linear combination of univariate ridge polynomials due to the following.  
\begin{lemma}\label{ridge polys are all polys}\citep[Proposition 5.19,][]{pinkus2015ridge}
Let $\Pi_{n}^{d}$ denote the space of polynomials of degree $d$ in $n$ variables. Then
	\begin{equation}
		\Pi_{n}^{d} = \mathrm{span}\{p(w^Tx):p\in\Pi_{1}^{d},\,w\in\R^{n}\} \label{poly approximation}.
		\end{equation}	 
\end{lemma}
Furthermore, it is known that polynomials have the following universal approximation property. 
\begin{lemma}\citep[Theorem 15.3, Corollary 4][]{treves2016topological}\label{thm: poly density}
	Polynomials are m-dense on compact subsets in $C^m(\R^{n})$ for any $m,n\in\mathbb{N}_+$.
\end{lemma} 
We can therefore apply Theorem \ref{thm: universality}. 
\begin{corollary}
	There exists a P-SympNet $\Phi_h^{\bar{H}^\theta}$ according to Definition \ref{def:psympnet} such that for any $H\in C^1(\Omega)$ and any $\epsilon>0$
	\begin{equation}
	\|\Phi_h^{\bar{H}^\theta}(x)-\phi_h^{H}(x)\|<\epsilon
	\end{equation}
	for  $x$ in some compact $\Omega \subset\R^{2n}$.
\end{corollary}   
\begin{proof}
	By Definition \ref{def:psympnet} and Lemma \ref{ridge polys are all polys} P-SympNets have an inverse modified Hamiltonian that spans all of $\Pi^{d}_{2n}$. As polynomials are dense on compact subsets in $C^1(\R^{2n})$ according to Lemma \ref{thm: poly density} the result follows from Theorem \ref{thm: universality}.
\end{proof}

% \textbf{Proposition 1.} \textit{Let $\phi_h^{H(x)}$ be a Hamiltonian diffeomorphism on a symplectic manifold $(\R^{2n}, \omega)$, corresponding the flow of a Hamiltonian ODE $\dot{x}=J\nabla H(x)$, where $H(x)\in\mathbb{P}_d(\R^{2n})$ is polynomial of degree $d$ in $2n$ variables. Then $\phi_h^H$ admits the following representation as the composition of shear flows of polynomial ridge functions 
% $$\phi_h^{H(x)} = \phi_h^{P_1(w_1\cdot x)}\circ\dots\circ \phi_h^{P_k(w_k\cdot x)}$$
% for some set of degree $d$ univariate polynomials $P_i \in \mathbb{P}_d(\R)$ and vectors $w_i\in\R^{2n}$.}
% \\

% Note that the Hamiltonian diffeomorphisms on $P_i$ take the form 
% \begin{equation}
% 	\phi_h^{P_i(w_i\cdot x)} = x + h P_i'(w_i\cdot x) J w_i, 
% \end{equation}

\begin{remark}
	In certain situations, P-SympNets could become unstable in the training process due to the unboundedness of the polynomial activation functions. This can be remedied by adding a L2 regularization term to the loss function and/or by composing the polynomial basis Hamiltonians with a sigmoidal function, e.g., setting $\alpha(w^Tx) = \sigma(p(w^Tx))$ in equation \eqref{scalar shear layer}. 
\end{remark}
\subsubsection{P-SympNets for Linear Systems}
Let's now turn our attention to the scenario where the data is generated by a linear Hamiltonian system $$\dot{x} = JMx.$$ The flow map of such an ODE is given by the matrix exponential of the Hamiltonian matrix $e^{hJM}$, which is symplectic. We will now show that this can be represented \textit{exactly} by a P-SympNet. In addition, the following theorem states that a P-SympNet can represent any symplectic matrix, not just those whose Lie algebra are Hamiltonian matrices. The proof can be found in Appendix \ref{app:proof of linear hamiltonian flows}. 
  
% We now recall the following result from \citep[Corollary 3.47]{hall2013lie}. 
% \begin{lemma}\label{hall lemma}
% 	If $G$ is a connected matrix Lie group with Lie algebra $\mathfrak{g}$, then any $A\in G$ can be represented by 
% 	\begin{equation}
% 		A=e^{hJX_1} e^{hJX_2} ... e^{hJX_k}
% 	\end{equation}
% 	for some $X_1,...,X_k\in\mathfrak{g}$.
% \end{lemma}

\begin{theorem}[Representation property]\label{thm: representation of linear hamiltonian flows}
	Let $S$ denote a linear symplectic transformation. Then there exists a $k$-layer P-SympNet $\Phi_h^{\bar{H}^\theta}$ such that 
	\begin{equation}
		S(x) = \Phi_h^{\bar{H}^\theta}(x).
	\end{equation}
	Moreover, the number of layers can be bounded as follows.
	\begin{itemize}
		\item[$(i)$] If $S$ is an arbitrary symplectic matrix, then $k\le 5n$. 
		\item[$(ii)$] If $S=\begin{pmatrix} A & B \\ C & D \end{pmatrix}$, where $\det (A)\ne 0$, then $k\le 4n$. 
		\item[$(iii)$] If $S =e^{hJM}$ where $M=M^T\in\R^{2n\times 2n}$ and $h>0$ sufficiently small, then $k\le 2n$. 
	\end{itemize}
\end{theorem}
We verify numerically the bound $(iii)$ in Section \ref{sec: linear hamiltonian experiments}. Note that for quadratic ridge functions then a P-SympNet is equivalent to expressing an element of the symplectic group in canonical coordinates of the second kind \citep{owren2001integration}.  

\subsection{Shear Hamiltonians Using Generalized Ridge Functions} \label{sec: generalised ridge functions}
Here, we outline how to construct shear Hamiltonians of more than one variable using generalized ridge functions. These are functions that project a $2n$ dimensional input onto a hyperplane by a linear transformation, e.g., $H(Cx)=K(Ap + Bq)$, $C=(A, B):\R^{n\times2n}$. For this generalized ridge function to be a shear, we need a method of parameterizing $A$ and $B$ whilst satisfying the symmetry condition \eqref{symmetry cond}. Once we have found such a parameterization a SympNet can be constructed by composing shear maps of the form
\begin{equation}\label{eq:shearSympNetlayer}
\phi^{\bar{H}^{\theta}_i}_h\left(
\begin{array}{c}
p\\q\\
\end{array}
\right)
=
\left(
\begin{array}{c}
p-hB \nabla \bar{H}^{\theta}_i(Ap+Bq)\\
q+hA \nabla \bar{H}^{\theta}_i(Ap+Bq)\\
\end{array}
\right),
\end{equation}
which is just equation \eqref{fe shear} written out in the coordinates $p$ and $q$. We can now define a SympNet that uses generalized ridge functions as follows. 
\begin{definition}[GR-SympNets]\label{def:grsympnet}
	A GR-SympNet is a map $ \Phi_h^{\bar{H}^\theta}(x)$ according to definition \eqref{def:SympNet} where 
	\begin{align}
		\mathcal{H} = \{\mathcal{N}(Ap + Bq) |\,\, AB^T=BA^T,\,\, A,B\in\R^{n\times n},\,\, \mathcal{N}:\R^n\rightarrow\R^n \}
	\end{align}
\end{definition}
That is, the layers are given by \eqref{eq:shearSympNetlayer} with $AB^T$ symmetric and $\mathcal{N}$ is a neural network, for example. 
%\subsubsection{Two degrees of freedom ($d=2$)}
%In two degrees of freedom $d=2$, for the matrices $A=[a_{ij}]$ and  $B=[b_{ij}]\in\R^{2\times 2}$ to satisfy the symmetry condition, we require the following matrix to vanish
%\begin{equation}
%	\left(\begin{array}{cc}
%	a_{11} & a_{12} \\
%	a_{21} & a_{22} \\
%	\end{array}\right)
%	\left(\begin{array}{cc}
%	b_{11} & b_{21} \\
%	b_{12} & b_{22} \\
%	\end{array}\right) 
%	- 
%	\left(\begin{array}{cc}
%	b_{11} & b_{12} \\
%	b_{21} & b_{22} \\
%	\end{array}\right)
%	\left(\begin{array}{cc}
%	a_{11} & a_{21} \\
%	a_{12} & a_{22} \\
%	\end{array}\right).
%\end{equation}
%The diagonal elements are identically zero, therefore this occurs when 
%$$ a_{11}b_{21}+a_{12}b_{22} =  b_{11}a_{21}+b_{12}a_{22}.$$
%This can be enforced by setting $a_{11}= b_{21}^{-1}( b_{11}a_{21}+b_{12}a_{22}-a_{12}b_{22})$ for $b_{21}\ne 0$, for example. 
%
%Alternatively, we can require that the matrices $A$ and $B$ are symmetric ($a_{12}=a_{21}$ and $b_{12}=b_{21}$) and require that $A$ and $B$ commute. Commutativity then occurs when the following ratios are equal 
%$$ \frac{a}{b} =  \frac{a_{11}-a_{22}}{b_{11}-b_{22}}$$
%or if diagonal elements are equal ($a_{11}=a_{22}$ and $b_{11}=b_{22}$).
%In either case, the symmetry condition \eqref{symmetry cond} will be satisfied. 

There are $\frac{3}{2}n(n-1)$ independent parameters to construct $A$ and $B$ satisfying $AB^T$ symmetric. In fact, the rows of $C$ form an isotropic subspace of the symplectic vector space $(\R^{2n}, \omega)$. The set of all such isotropic subspaces is called the Lagrangian Grassmanian, which is a compact manifold homeomorphic to $U(n)/(O(k)\times U(n-k))$ for $0<k\le n$. Thus, one could choose points in the Lagrangian Grassmanian to parameterize $A$ and $B$, see \citep{mclachlan2004explicit} for a more detailed discussion. We will instead adopt the following for the parameterization of $A$ and $B$ that uses $\frac{3}{2}n(n+1)$ free parameters to satisfy the symmetry condition and requires that either $A$ or $B$ be invertible. 
\begin{proposition}\label{thm:GR basis}
	Assume $A$ (or $B$) is invertible and let $S_i=S_i^T$ for $i=1,2,3$ be symmetric matrices. Then
	\begin{align}
		A \text{ (or $B$)} &= I + S_3S_2,\\
		B \text{ (or $A$)} &= S_3 + S_3 S_2 S_1 + S_1,
	\end{align}
	satisfies $AB^T=BA^T$. 
\end{proposition}
This is proved in Appendix \ref{sec:proof of gr thm}. In practice, the layers alternate between the two choices. The symmetric parameters of $S_i$ are what we optimize over in addition to the weights and biases of the neural network. We mention several alternative choices for $A$ and $B$ that satisfy the symmetry condition. The first is $A=\mathrm{diag}(a)$, $B=\mathrm{diag}(b)$ for vectors $a,b\in\R^{n}$ and is used in \cite{mclachlan2004explicit,feng1998variations} for constructing symplectic splitting methods for polynomial ODEs. Another is $A=EQ\mathrm{diag}(a)Q^T$,  $B=EQ\mathrm{diag}(b)Q^T$ for $Q\in O(n)$, $E\in\R^{n\times n}$. The latter choice spans all matrices where $AB^T$ is symmetric. However, we have found empirically that the parameterization from proposition \ref{thm:GR basis} yields slightly better results tho we remark that the latter alternative choice could be more general. 

Note that GR-SympNets are a super set of the G-SympNets proposed in \cite{jin2020sympnets}. This can be seen by setting $A=I$ and $B=0$ (and vice versa), i.e.,  $S_1=S_2=S_3=0$. Therefore, GR-SympNets inherit the same density properties as G-SympNets according to the following theorem that is proved in Appendix \ref{proof of grsympnet density}.
 
\begin{theorem}\label{thm:grsympnet density}
	Let $Sp^r(\Omega)$ be the space of $r$-finite symplectomorphisms on some open compact subset $\Omega\subset\R^{2n}$. Then there exists a GR-SympNet $\Phi_h^{\bar{H}^\theta}$ according to Definition \ref{def:grsympnet} such that for any $\phi_h(x)\in Sp^r(\Omega)$ and any $\epsilon>0$
	\begin{equation}
	\|\Phi_h^{\bar{H}^\theta}(x)-\phi_h(x)\|<\epsilon
	\end{equation}
	for $x\in \Omega $.
\end{theorem}

For the one degree of freedom case $n=1$, the matrices $A=a$, $B=b\in\R$ are scalars and therefore satisfy the symmetry condition without requiring Proposition \ref{thm:GR basis}.

Lastly, we remark that GR-SympNets are essentially making a generalized ridge function approximation to the inverse modified Hamiltonian of the form $\bar{H}^{\theta} = \sum_{i=1}^k \bar{H}^{\theta}_i(C_ix)$, where $C_i$ are constrained by Proposition \ref{thm:GR basis}. It's remains to show that the span of these functions are dense in $C^1(\Omega)$. 
%differen
%Therefore, if one could show that the only vanishing polynomial in $L(\mathcal{T})$ is the zero polynomial, then using Theorem \ref{thm:generlised ridge function} the above approximation is dense in $C(\R^n, \R)$. This means that a GR-SympNet can approximate the exact flow of \textit{any} inverse modified Hamiltonian and this can be used to show density in the group of symplectic maps. However, we will leave this task to a reader more competent in functional analysis. 

%For the polynomial case, things become simpler and we can further leverage theorems 4-7 from \citep{feng1998variations}. To summarize these results, when $H(p,q)$ is a polynomial of degree $n$ in the $2d$ variables $p$ and $q\in\R^{d}$ then it can be expanded as a sum of shear polynomial Hamiltonians $$H(p,q)=\sum_{i=1}^{m}P_i(A_ip + B_iq),$$ where $P_i$ is a polynomial of $d$ variables and $A_i$ and $B_i$ are diagonal matrices and therefore commute. When $d=1$ we have the upper bound $m\le n+1$.

\subsection{Shear Hamiltonians Using Fixed Direction Ridge Functions} \label{sec: shear hamiltonians depending on p or q}
In this section we discuss the case where $A=I$, $B=0$ or $A=0$, $B=I$, which corresponds to shear Hamiltonians of $p$ or $q$ only. Shear Hamiltonians of this kind naturally arise from splitting methods applied to separable Hamiltonian vector fields $H(p,q)=T(p) + V(q)$, such as the symplectic Euler method and its higher-order variants. This setting is what has been used extensively in the literature to create symplectic neural networks and deep Hamiltonian neural networks such as those in \cite{jin2020sympnets,chen2019symplectic,burby2020fast,tong2021symplectic,zakwan2023universal,maslovskaya2024symplectic} that implements alternating composition of the following horizontal and vertical shear maps
\begin{align}\label{v and h shears}
\begin{split}
\phi^{T_i}_h \left(\begin{array}{c}
p\\q\\
\end{array}\right)
=&
\left(\begin{array}{c}
p\\ 
q + h\nabla T_i(p)\\
\end{array}\right), \\
\phi^{V_i}_h \left(\begin{array}{c}
p\\q\\
\end{array}\right)
=&
\left(\begin{array}{c}
p - h\nabla V_i(q)\\ 
q\\
\end{array}\right).
\end{split}
\end{align}
Taking alternating compositions of the above yields
\begin{equation}\label{sep sympnet}
	\Phi_h^{T+V} = \phi^{T_k}_h\circ\phi^{V_k}_h\circ...\circ\phi^{T_1}_h\circ\phi^{V_1}_h
\end{equation}
which give rise to maps that are dense in the space of symplectic flows as shown recently in \cite{berger2022generators} hence this serves as good theoretical underpinning and motivation for the methods in the aforementioned references. Although as mentioned in Section \ref{sec: previous work and motivation}, one requires tools from functional analysis to assert such claims and these architectures are less amenable to techniques from numerical analysis which requires power series to converge. 

To elucidate this concept, consider the inverse modified Hamiltonian of the map \eqref{sep sympnet}
\begin{equation}
	\bar{H} = \sum_{i=1}^k \left({T}_i(p) + {V}_i(q)\right),
\end{equation}
which is separable. However, according to equation \eqref{invmod bch}, this means that the true Hamiltonian must be of the form 
\begin{equation}
	H(p, q) = \sum_{i=1}^k \left({T}_i(p) + {V}_i(q)\right) + O(h),
\end{equation}
which is $O(h)$ away from a separable Hamiltonian assuming the first few terms of the BCH series converges. Hence, if the true Hamiltonian is not well approximated by a separable Hamiltonian, then $T_i(p)$ and $V_i(q)$ must necessarily be learned such that the BCH series does not converge. For example, consider the simplest case of composing two linear flows $e^{hA}e^{hB}$. Then its BCH series is absolutely convergent only when $\|hA\| + \|hB\| \le \ln(2)$ \citep{biagi2020baker} meaning the norms of these matrices must be large for the BCH series to not converge. We show similar results for non-separable Hamiltonians in the numerical experiments sections considering the magnitude of the learned parameters in Section \ref{sec: learned parameters}.

\subsubsection{Symplectic Splitting Methods}
The symplectic Euler method can be written as the composition $$\Phi_h(x) = \phi^{V}_h\circ \phi^{T}_h(x).$$ The order-two symmetric Stormer-Verlet scheme can then be obtained by symmetric composition (Strang splitting) $\Phi_{h/2}^*\circ \Phi_h(x)$ or $\Phi_{h/2}\circ \Phi_h^*(x)$. Higher order maps can be constructed by applying multiple layers with carefully selected time steps. Such maps are used to construct symplectic recurrent neural networks in \cite{chen2019symplectic}. In general, a splitting method for a separable system is given by \citep{hairer2006geometric}
\begin{equation}
\Phi_h := \phi^{T}_{\alpha_1 h} \circ\phi^{V}_{\beta_1 h}
\circ 
\dots
\circ \phi^{T}_{\alpha_{k} h} \circ\phi^{V}_{\beta_{k} h}
\end{equation}
where $\sum_i \alpha_i = 1$ and $\sum_i \beta_i = 1$ for consistency, and high order methods can be constructed by increasing $k$, which would be well suited for separable Hamiltonians. We do not consider these methods as they are a subset of the following G-SympNets. 

\subsubsection{G-SympNets}
Two classes of SympNets were introduced in \cite{jin2020sympnets} which we will now outline in the framework of separable Hamiltonian shear flows.  

The first, named G-SympNets, use the compositions 
\begin{equation}\label{eq:gSympNet}
\Phi_h := \phi^{T_k}_h \circ\phi^{V_{k-1}}_h
\circ 
\dots
\circ \phi^{T_{2}}_h \circ\phi^{V_1}_h
\end{equation}
where $T_i(p)$ and $V_i(q)$ are parameterized by neural networks with activation functions given as the anti-derivative of an $r$-finite sigmoidal function $\int \sigma$. G-SympNets are dense in the space of symplectomorphisms \citep[Theorem 5]{jin2020sympnets}.

Note that G-SympNets are similar to the aforementioned symplectic splitting methods and differ only in the implementation of the gradient and the choice of splitting. In particular, the symplectic splitting methods are a subset of the G-SympNets and they are equal when $\alpha_i T_i=T$ and $\beta_i V_i=V$ hence G-SympNets are more general. 

G-SympNets have a separable inverse modified Hamiltonian of the form 
\begin{equation}
\bar{H}^{\theta} = \sum_{i~\text{even}}w_i^T ({\int} {\sigma}) (W_ip+b_i)+ \sum_{i~\text{odd}}w_i^T ({\int} {\sigma}) (W_iq+b_i),
\end{equation}
which are not known to be dense in $C^1(\Omega)$ and hence Theorem \ref{thm: universality} doesn't apply. 

\subsubsection{LA-SympNets}
The second SympNet proposed in \cite{jin2020sympnets} is constructed from alternating compositions of \textit{linear} layers and \textit{non-linear} activation layers, which are defined as follows. A linear layer is 
\begin{equation}\label{eq:laSympNet}
\Phi_h^{L_i} := \phi^{T_{k,i}}_h \circ\phi^{V_{{k-1},i}}_h
\circ 
\dots
\circ \phi^{T_{2,i}}_h \circ\phi^{V_{1,i}}_h
\end{equation}
where $T_{i,j}(p) = p^TA_{i,j}p$ and $V_{i,j}(q) = q^TA_{i,j}q$ where $A_{i,j}\in\R^{n\times n}$ are trainable symmetric matrices. Given sufficiently large $k$, a linear layer is shown to approximate an arbitrary linear symplectic map. 

An \textit{activation} layer is a shear map  $\phi_h^{T_i}$ or $\phi_h^{V_i}$ defined by the shear Hamiltonians $T_i(p)=a_i^T(\int \sigma ) (p)$ or $V_i(q)=a_i^T(\int \sigma ) (q)$, where  $a_i\in R^n$ is a trainable vector of parameters and $\sigma$ is an activation function applied point-wise to its argument.  The LA-SympNet is then defined by alternating flows of linear and activation layers
\begin{equation}
\Phi_h^{LA} :=
\Phi_h^{L_1} \circ \phi_h^{T_1} \circ \Phi_h^{L_2} \circ  \phi_h^{V_2}\circ\dots\circ\phi_h^{V_k}\circ\Phi_h^{L_{k+1}} 
\end{equation}
This form of SympNet takes inspiration from the layers of standard feed-forward deep neural network, which are the composition of non-linear activation functions applied point-wise to affine transformations. LA-SympNets are also dense in the space of symplectomorphisms \citep[Theorem 4]{jin2020sympnets}.

LA-SympNets can approximate an inverse modified Hamiltonian of the form 
\begin{equation}
\bar{H}^{\theta} = \sum_{i=1}^{\mathrm{ceil}(k/2)} (a^T_i\sigma(p)+b^T_i\sigma(q)) + p^TAp + q^TBq
\end{equation}
which are not known to be dense in $C^1(\Omega)$ and hence Theorem \ref{thm: universality} doesn't apply. 
\subsubsection{Hénon-like Maps (H-SympNets)}
In \citep{turaev2002polynomial}, a universality result is derived for polynomial Hénon-like maps of the form 
\begin{align}\label{polyhenon}
\psi^{\alpha_i}_{h}\left(\begin{array}{c}
p\\q\\
\end{array}\right)
=
\left(\begin{array}{c}
q \\ 
-p + h\nabla\alpha_i(q)\\
\end{array}\right),
\end{align}
where $\alpha_i:\R^n\rightarrow\R$ is a polynomial mapping of some degree. The main result of this study is that compositions of maps of the form $\psi^{\alpha_i}_{h}$ are dense in the space of symplectic maps and therefore are also universal approximators for arbitrary Hamiltonian systems. There is a clear connection between the maps $\psi^{\alpha_i}_{h}$, $\phi^{T_i}_{h}$ and $\phi^{V_i}_{h}$ which can be seen by 
\begin{equation}\label{henon and shears}
\psi^{\alpha_i}_{h}=
\phi^{\alpha_i}_{h} \circ \phi^{V}_{h}\circ \phi^{-T}_{h}\circ \phi^{V}_{h}
\end{equation}
where $V(q)=\frac{1}{2}\|q\|^2$ and $T(p)=\frac{1}{2}\|p\|^2$, meaning any Hénon-like map is expressible as a composition of shears. Hénon-like maps are dense in the space of symplectomorphisms \citep{turaev2002polynomial}. 

\section{Numerical Experiments}\label{sec: experiments}
\subsection{Experimental Setup}
\textit{Data generation.} The training data sets are of the form $\mathcal{D}=\{x_i, \phi^{H}_h(x_i)\}_{i=1}^{N_{\text{data}}}$ with constant time-step $h$, where $\phi^{H}_h(x_i)$ is approximated to high precision using \texttt{scipy.integrate.odeint} with a tolerance of \texttt{1e-15} to solve the Hamiltonian ODE \eqref{hode}. We note that SympNets work just as well on data sets with irregularly spaced data \citep{jin2020sympnets}, however we won't explore this scenario here. In addition, in \cite{jin2020sympnets}, the training data in the form of one discretely sampled trajectory is used. However, we have noticed that the particular initial condition used makes a large difference to the relative performance of the models. To mitigate choosing a particular trajectory that favors a one model over another, we randomly sample $x_i$ from a uniform distribution in the unit hypercube $[-\frac{1}{2}, \frac{1}{2}]^{2n}$. The test data sets are generated in the exact same way, and are of equal size to the training set. 

\textit{Hyperparamter selection.}
To mitigate serendipitous selection of model architectures and to test robustness with respect to hyperparameter choice we will test all the SympNets over a wide range of hyperparameters. We will train SympNets with layers $k \in \{8, 12, 16, 24, 32, 48, 64\}$ and widths
$m \in \{4, 8, 16, 32, 64\}$ (for the G, GR, H, R-SympNets), maximum degrees
$d \in \{2, 3, 4, 8, 12, 16, 24\}$ (for the P-SympNets), or sub-layers
$k_{sub} \in \{3, 6, 9, 12\}$ (for the LA-SympNets). We consider the LA, G and H-SympNets as ``benchmark'' methods, as these appear to be the most popular SympNets and have been used extensively in the literature.  

\textit{Training.}
Each model is trained for 50,000 epochs with a learning rate of $0.002$ with the Adams optimizer. We recall the loss is given as $\mathrm{loss}=\sum_{x_i\in \mathcal{D}}\|\Phi_h(x_i)-\phi_h^H(x_i)\|^2$ where $\Phi_h$ is the SympNet and $\mathcal{D}$ is the training set. 

\textit{Evaluation.}
To test how the models generalize to unseen data, we evaluate the same loss function over the test set. The best training losses achieved by each model are also reported. We would also like to measure the computational cost of the training the SympNets. We do this by simply timing the models over the 50,000 epochs, although we note that the total training time alone doesn't take into account the convergence rate of the training curve. Note that the training time is proportional to the number of parameters in the model, so displaying the losses versus the number of parameters gives similar figures. 

\textit{Implementation and code.}
We have implemented all the SympNets in a PyTorch framework that can can be cloned from \url{github.com/bentaps/strupnet}. Alternatively, the models have been packaged and distributed on PyPI such that they can be installed via pip, i.e., \texttt{pip install strupnet}. The following is an example of how to initialize a model.   
\begin{center}
	\begin{minipage}{0.6\textwidth}  
\lstset{style=vscode-dark} 
\begin{lstlisting}[language=Python, caption= Initialising a P-SympNet in Python. The line \texttt{x1=sympnet(x0, timestep)} implements a symplectic transformation according to Definition \ref{def:psympnet} with random weights and direction vectors.]
import torch
from strupnet import SympNet

dim = 2
sympnet = SympNet(
    dim=dim, 
    layers=12, 
    max_degree=8, 
    method="P",
)
timestep = torch.tensor([0.1])
x0 = torch.randn(2 * dim) 
x1 = sympnet(x0, timestep) 
\end{lstlisting}
\end{minipage} 
\end{center}
\subsection{Separable polynomial Hamiltonians} 
 
\subsubsection{The Hénon-Heiles system}
Here we consider the Hénon-Heiles Hamiltonian, which is separable and polynomial with $n=2$ degrees of freedom given by
$$H = \frac{1}{2} (p_1^2+p_2^2+q_1^2+q_2^2) + q_1^2 q_2 + q_2^3.$$
Two datasets are generated with $n_{\text{data}}=100,200$ and $h=0.01, 0.1$, respectively. The train and test set errors versus the training time of the SympNets are plotted in figure \ref{henon expt}. In all of our experiments the training time is positively correlated to the number of parameters in the model, i.e., models with more parameters take longer to train.

The most striking observation here is that the P-SympNet consistently achieves test and training errors several orders of magnitude better than the other methods for equal or lower cost. The only exception are the P-SympNets of degree $d=2$, which can only represent linear dynamics and can be seen by the four yellow dots at the top of the figures. 

In figure \ref{henon a}, we find two P-SympNets with a training and test accuracy around $10^{-15}$ and lower corresponding to hyperparameters of $(k, d) = (8,3)$ and $(8,4)$. In fact, out of all the P-SympNets with $d>2$, these are the two with the least number of free parameters (48 and 56, respectively), indicating that the other models are over parameterized. Note that the central idea of our proposed SympNets is the that they can better approximate the true inverse modified Hamiltonian. In the present scenario, the time step $h=0.01$ is reasonably small hence the true inverse modified Hamiltonian is close to the true Hamiltonian, which is a degree $d=3$ polynomial expressed as the sum of six monomials, which is well approximated by a P-SympNet with $k=8$ layers of $d=3$.  

In the $h=0.1$ experiment (figure \ref{henon b}), we notice a different trend with the P-SympNets. Here, the inverse modified is less accurately approximated by a cubic polynomial due to the fact that the higher order terms in the BCH expansion \eqref{true ham from inv mod ham} now contribute more significantly to the overall expression. This is reflected by the fact that when we increase the number of parameters, including the degree of P-SympNet, the accuracy improves. 

The GR, G, H and R-SympNets perform equally well, and it is difficult to distinguish them in terms of their performance in this experiment. One thing to note is that these methods do not improve their accuracy when the hyperparameters increase, which could mean that the architectures are difficult to optimize. 

\begin{figure} 

	\centering  
	\begin{subfigure}{0.32\linewidth}
		\centering
		\includegraphics[width=\linewidth]{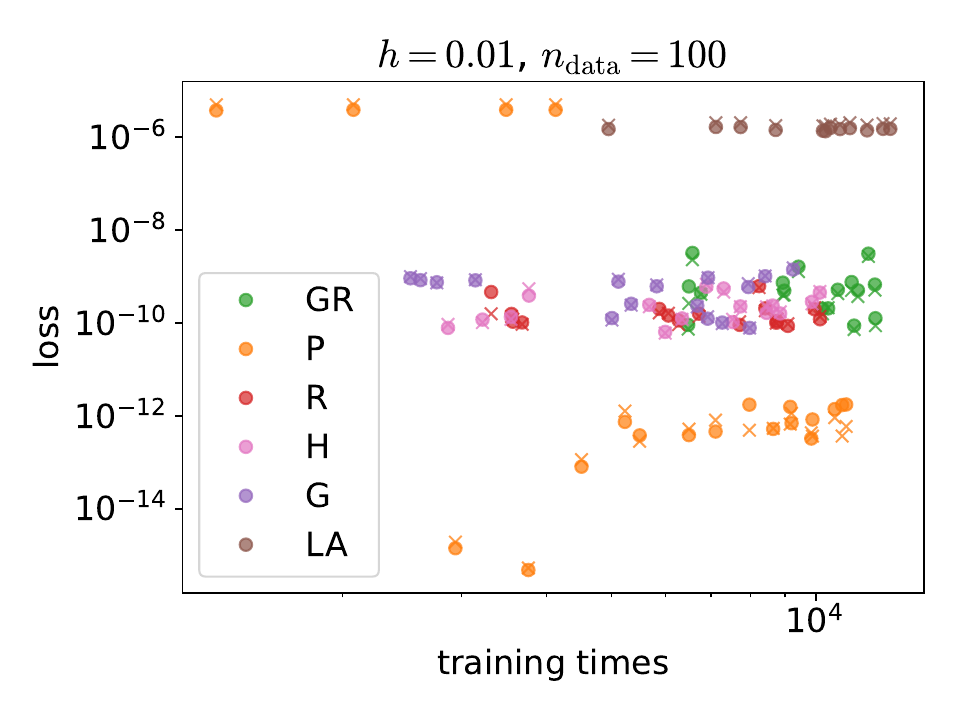}
		\caption{Smaller timestep.}\label{henon a}
	\end{subfigure}
	\begin{subfigure}{0.32\linewidth}
		\centering
		\includegraphics[width=\linewidth]{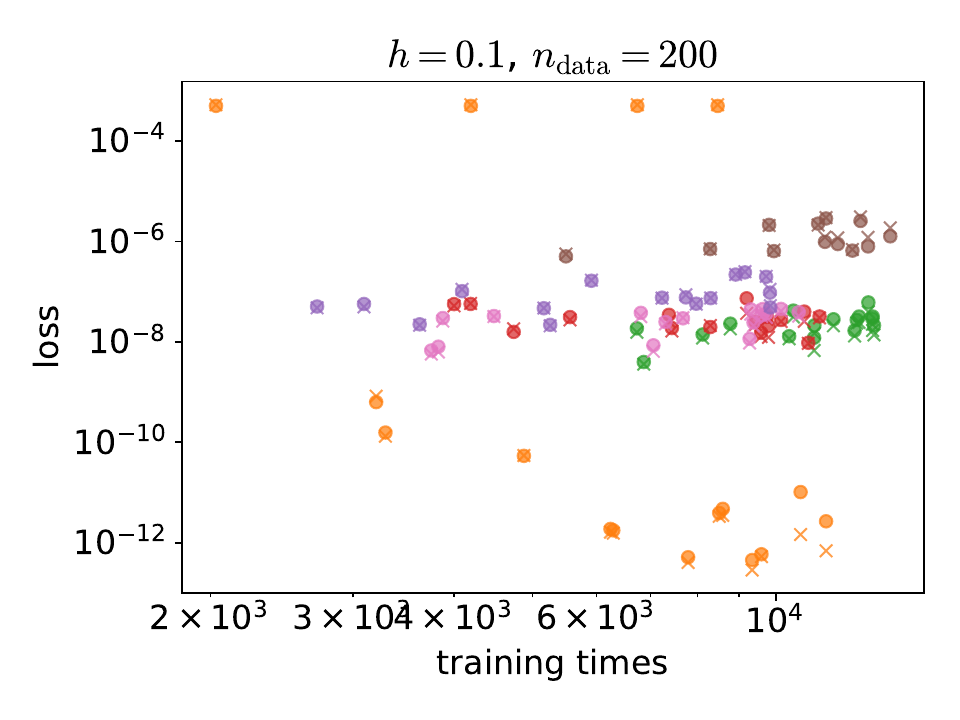}
		\caption{Larger timestep.}\label{henon b}
	\end{subfigure}
		\vskip -0.1in
	\caption{The Hénon-Heiles system experiments. Each point plots the training set errors (crosses) and test set errors (dots) plotted against training time for different choices of hyperparameters.} \label{henon expt}
	\vskip -0.1in
\end{figure}

\subsubsection{Fermi-Pasta-Ulam system}
We also perform a similar experiment on the Fermi-Pasta-Ulam system, which is a separable, polynomial Hamiltonian with $n=4$ degrees of freedom given by
$$H = \sum_{i=1}^{n} \frac{1}{2}p_i^2 + \frac{1}{2}(q_{i+1}-q_i)^2+\frac{1}{4}(q_{i+1}-q_i)^4,$$
where $q_{n+1}:=0$. We train the SympNets on a data set with $n_{\text{data}}=400$ and $h=0.01$. The results are presented in figure \ref{fpu expt}. Like the Hénon-Heiles system, the P-SympNets consistently outperform the other methods other than the P-SympNets with quadratic ridge functions, which make a linear approximation to the solution of the Hamilton ODE. 
\begin{figure}
	\centering  
	% \begin{subfigure}{\linewidth}
	% 	\centering  
	% 	\includegraphics[width=0.32\linewidth]{results/hyperparams/fpu4/experiment-0/figs/training_times_fpu4_nsols=200_h=1.0e-02.pdf}
	% 	\caption{} 
	% \end{subfigure} 
	\begin{subfigure}{\linewidth}
		\centering
		\includegraphics[width=0.32\linewidth]{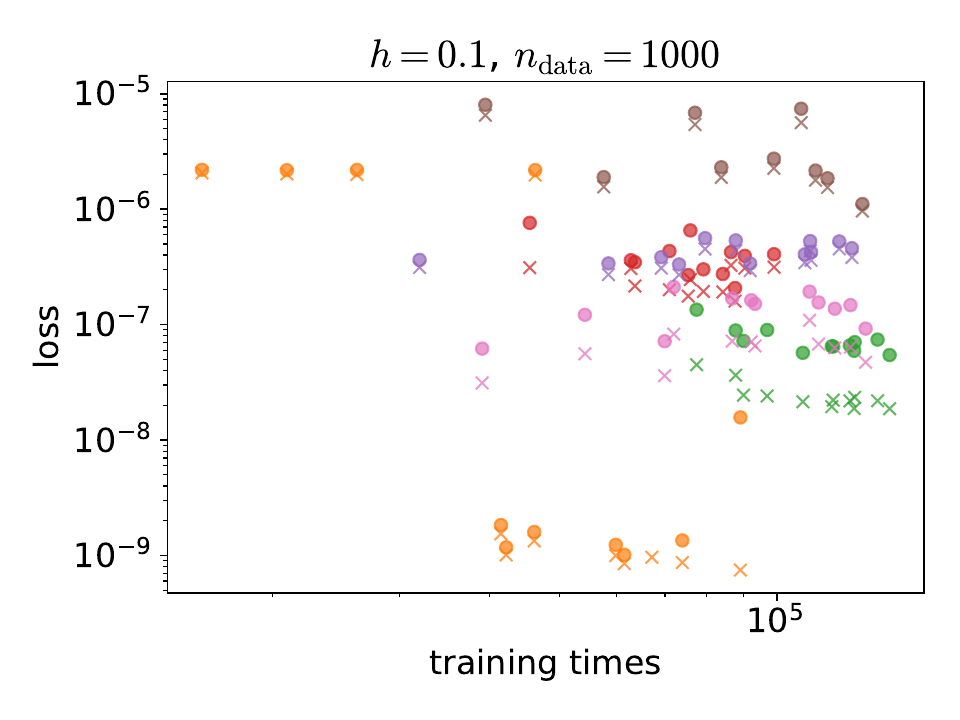}
	\end{subfigure}
	\vskip -0.1in
	\caption{Fermi-Pasta-Ulam experiment. The training set errors (crosses) and test set errors (dots) each point represents a different choice of hyperparameters.} \label{fpu expt}
	\vskip -0.1in
\end{figure}

\subsection{Linear Hamiltonian systems}\label{sec: linear hamiltonian experiments}
One goal of this experiment is to test how the SympNets perform when the data is generated by a linear Hamiltonian system. In particular, the representation Theorem \ref{thm: representation of linear hamiltonian flows} for P-SympNets suggests that we can find an exact parameterization of the true map. To this end, we consider higher-dimensional data sets generated by a quadratic Hamiltonians. 

\subsubsection{A dense linear system}
Here we consider a data set corresponding to a dense linear Hamiltonian system in $2n=20$ dimensions of the form
$$H= \frac{1}{2}x^TAx$$
where $A=I+S\in\R^{{2n}\times {2n}}$, where $S_{ij}=S_{ji}$ are uniformly distributed between $[0,1]$. In this situation, we expect a P-SympNet to be able to parameterize the exact solution with $k=20$ layers. 

Indeed, this is confirmed by the errors from figure \ref{linear20}. Here we see the errors of the P-SympNet are often on the order of $10^{-20}$. Next, we train several P-SympNets with $k=10,11,...,30$ layers and plot the best test and training errors in figure \ref{linear20 layers}. We see that the P-SympNets with $k \ge 20$ layers learn the exact solution to machine precision, supporting the results Theorem \ref{thm: representation of linear hamiltonian flows}.

We mention that the LA SympNets also do a remarkable job here, consistently reaching test losses of $O(10^{-15})$. As before, the G, H and R-SympNets all reach roughly the same errors regardless of our choice of hyperparameters. 

Next, we train 20 quadratic P-SympNets with $d=2$ and $10 \le k \le 30$ on the same data set and plot the best test and training losses as a function of the number of layers $k$. In figure \ref{linear20 layers} we see that the P-SympNets with $k\ge 20$ layers learn the map to machine precision as predicted by Theorem \ref{thm: representation of linear hamiltonian flows} $(iii)$.

\subsubsection{A wave-like Hamiltonian system}

Next we consider the wave-like Hamiltonian system in $2n=100$ of the form 
 $$  
H = \frac{1}{2}p^Tp + \frac{1}{2}q^TAq
$$
where $A=A^T\in\R^{50\times 50}$ is a Laplacian matrix with $A_{i,i}=-2$ on the main diagonal and $A_{i, i+1}=A_{i+1, i}=1$ on the off-diagonals. We train several quadratic P-SympNets with $d=2$ with layers $k=20, 40, ..., 200$. The train and test losses are plotted as a function of the number of layers $k$ in figure \ref{linear50}. We see that the P-SympNets with $k\ge 100$ learns the exact solution very high precision, which again supports Theorem \ref{thm: representation of linear hamiltonian flows} $(iii)$.

\begin{figure}   
	\centering 
	\begin{subfigure}{0.32\linewidth}  
        \centering
		\includegraphics[width=\linewidth]{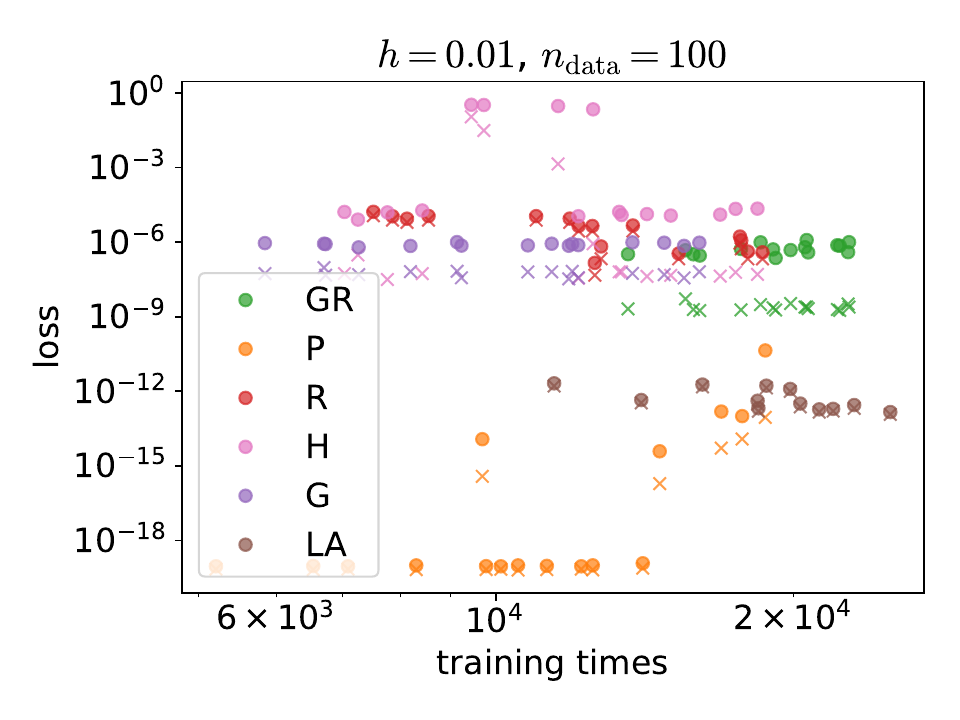}
		\caption{The dense linear system.} \label{linear20}	
	\end{subfigure}
	\begin{subfigure}{0.32\linewidth}
		\centering
        \includegraphics[width=\linewidth]{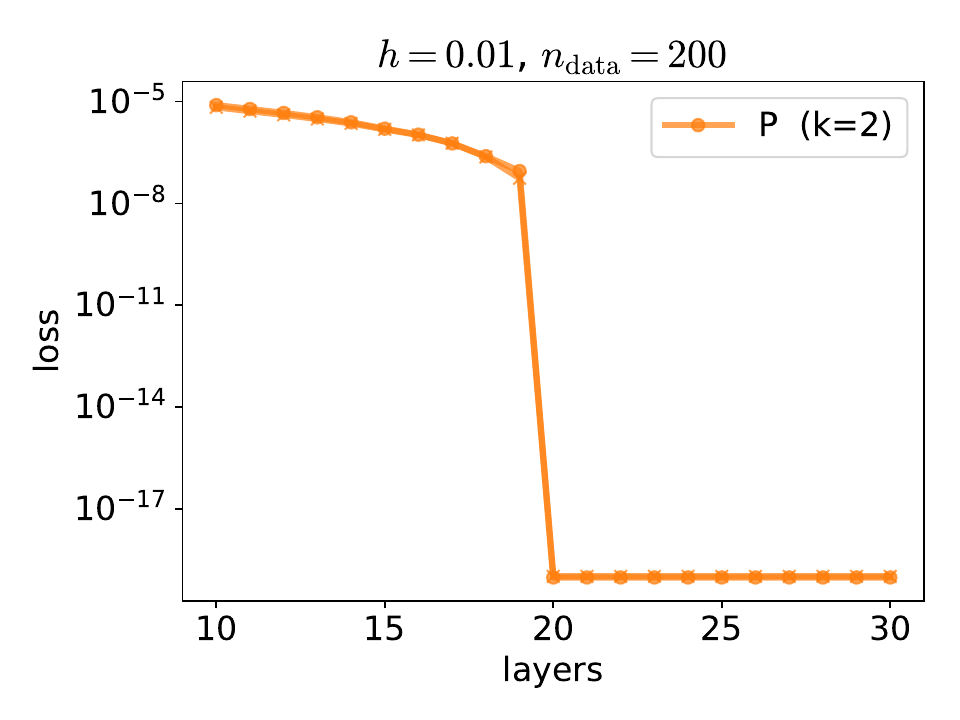}
		\caption{The dense linear system.} \label{linear20 layers} 
	\end{subfigure}
		\begin{subfigure}{0.32\linewidth}
		\centering
        \includegraphics[width=\linewidth]{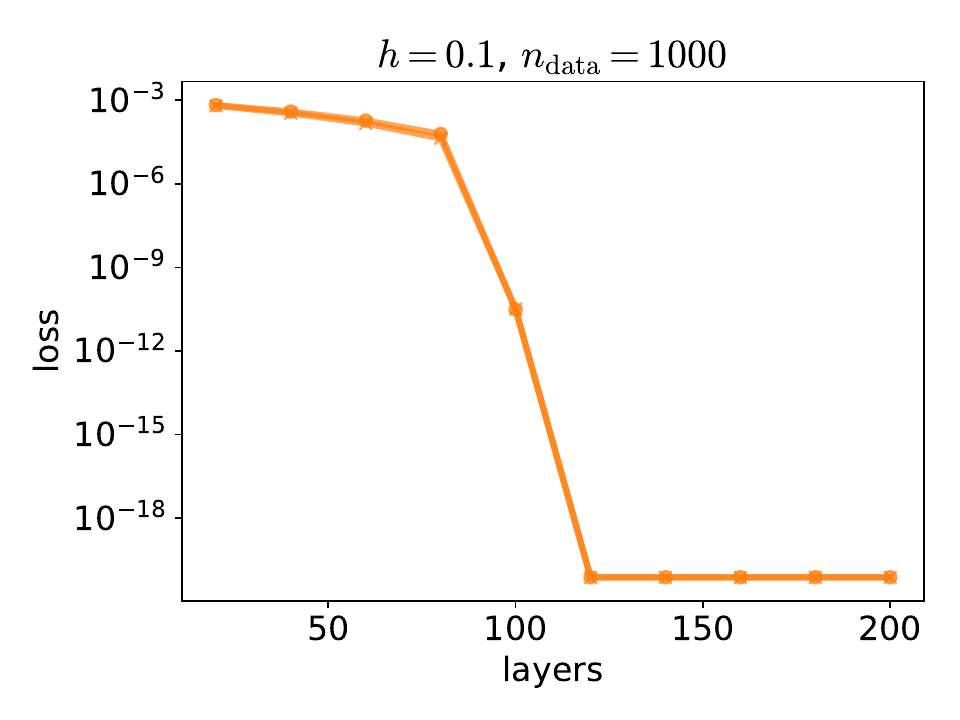}
		\caption{The wave-like system.} \label{linear50}
	\end{subfigure}
	\caption{ Figures (b) and (c) are the train and test set errors for the linear Hamiltonian systems plotted against the number of layers of a P-SympNet. Note that the training and test losses are roughly the same values for each model and are indistinguishable at this scale.} \label{linear expt}
\end{figure}

\subsection{The double pendulum}

We now look at three data sets generated by the double pendulum Hamiltonian, with
$$H = \frac{p_1^2 + 2 p_2^2 - 2 p_1 p_2 \cos(q_1 - q_2)}{2 (1 + \sin^2(q_1 - q_2))} - 2 \cos(q_1) - \cos(q_2).$$
This is a non-quadratic, non-separable, non-polynomial Hamiltonian with $n=2$ degrees of freedom. For these reasons, and the fact that this system exhibits chaotic motion, we consider this to be a one of the most challenging dynamical systems to learn. The results are presented in figure \ref{double pendulum expt}.

A common observation among all three experiments here is that the P-SympNets consistently achieve lower errors for lower cost, followed by the GR-SympNets. In addition, we observe a positive correlation between training time and error, meaning that the P-SympNet, and the GR-SympNet to a lesser extent, effectively use their parameters to better approximate the inverse modified Hamiltonian. This is in stark contrast to the almost all the other methods, which do not improve much when the number of parameters increases.  Remarkably, the P-SympNet achieves test set errors of order $10^{-7}$ for large time steps of $h=1$, meaning that it approximates the map $\phi_1^H$ to about 3 decimal places on this domain. 

Furthermore, we notice now that the G and LA-SympNets are not able to learn effective solutions to the true map, even for small time-steps. This could be due to the fact that the inverse modified Hamiltonian is not separable as we suggested earlier. 

\begin{figure}
	\centering    
    \begin{subfigure}{0.32\linewidth}
        \centering
        \includegraphics[width=\linewidth]{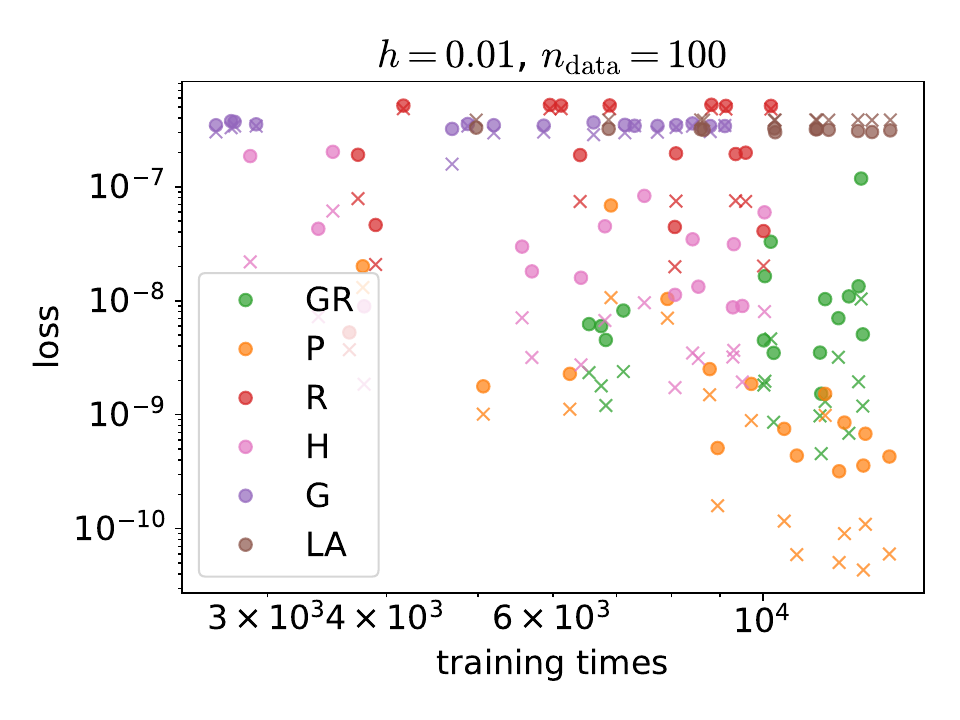}
        \caption{}\label{double pendulum a}
	\end{subfigure}
	\begin{subfigure}{0.32\linewidth}
		\centering
		\includegraphics[width=\linewidth]{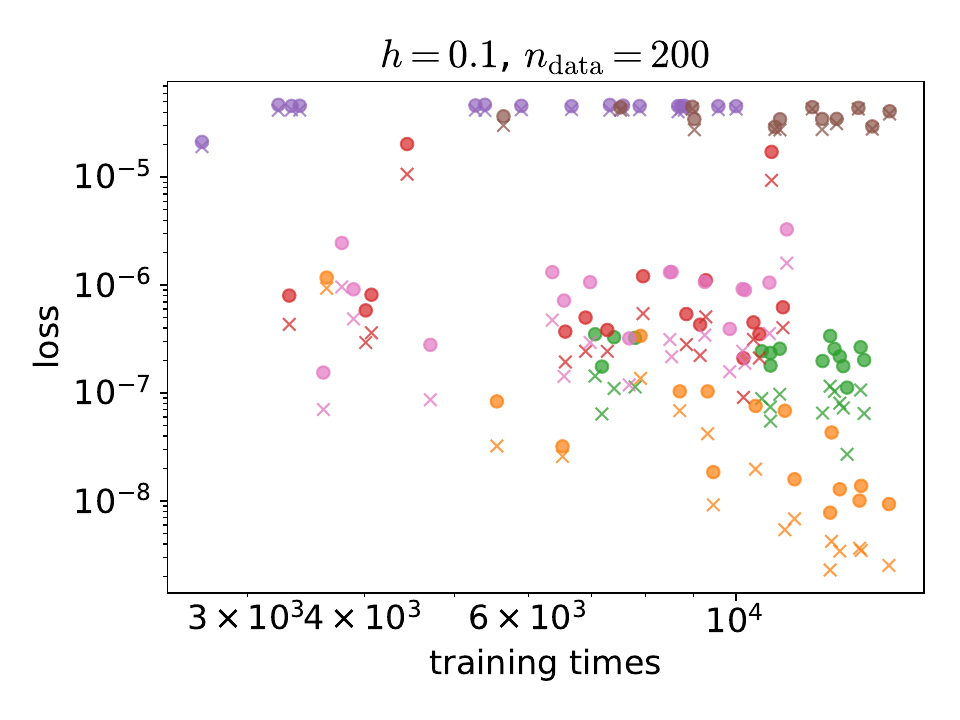}
		\caption{} 
		\end{subfigure}
	\begin{subfigure}{0.32\linewidth}
		\centering
		\includegraphics[width=\linewidth]{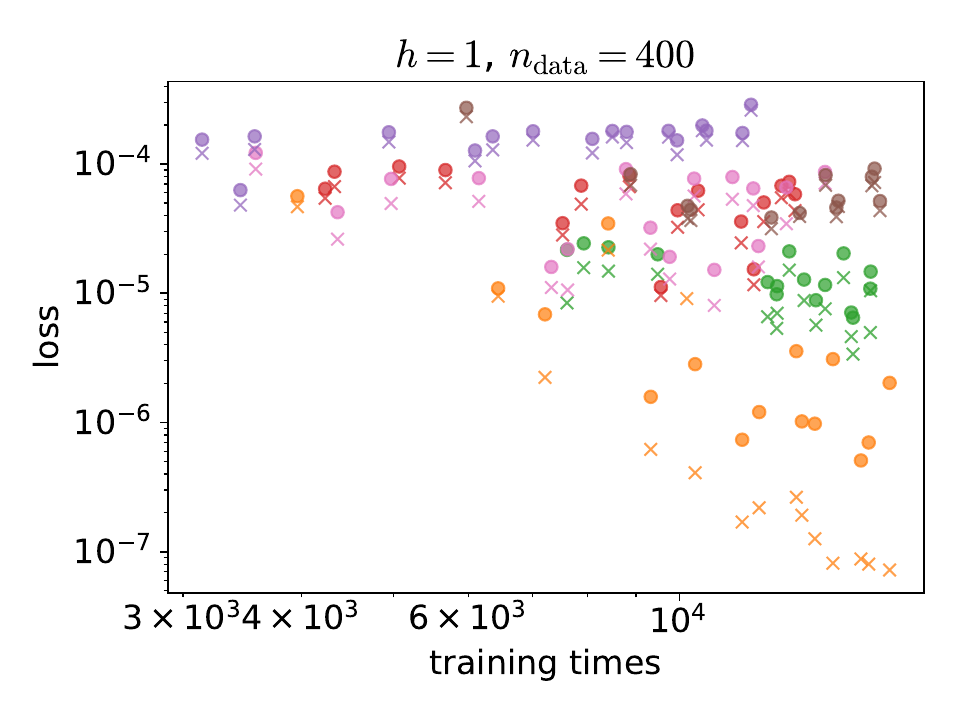}
		\caption{}
    \end{subfigure}
	\vskip -0.1in
	\caption{Double pendulum experiments. The training set errors (crosses) and test set errors (dots) each point represents a different choice of hyperparameters.} \label{double pendulum expt}
	\vskip -0.1in
\end{figure}

\subsection{Learned parameters}\label{sec: learned parameters}
One of the advantages of using P and R-SympNets is that according to Lemma \ref{lemma: small basis} $(ii)$, the SympNets can learn parameter values close to zero. To investigate the validity of this claim, we will look at the distribution of parameters that the SympNets learn in some previous experiments.

In figure \ref{fig: params} we plot the average values of the parameters that the SympNets learned for one of the double pendulum experiments, the dense linear Hamiltonian experiment and one of the Hénon-Heiles experiments. That is, two non-separable Hamiltonian systems and one separable system. The error bars denote the maximum and minimum values of the learned parameters. The main observation here is that the SympNets that use separable flows, namely the LA, G and H-SympNets require learning relatively large values. Whereas the P and R-SympNets learn smaller value. However, this difference between the models is not as pronounced for separable Hamiltonians as can be seen by figure \ref{henon params}. An exception to this trend are LA-SympNets, which learns small parameter values for the double-pendulum experiment but a poor approximation to true flow as seen by the large training and test sets errors in figure \ref{double pendulum a}. This could be due to the fact that the parameters are initialized by random values close to zero. We suspect that if the optimal parameters are found such that the LA-SympNet learns a map with the same test and train loss as the other models, then the learned parameters would be very large. 

\begin{figure}
	\centering 
    \begin{subfigure}{0.32\linewidth}
        \centering
        \includegraphics[width=\linewidth]{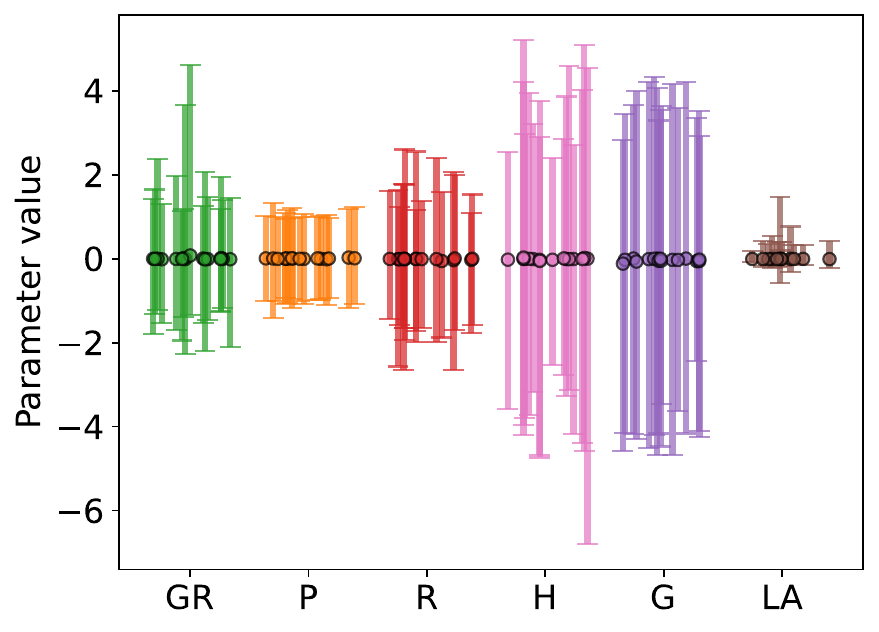}
        \caption{Double pendulum experiment from figure \ref{double pendulum a}.}
	\end{subfigure}
    \begin{subfigure}{0.32\linewidth}
        \centering
        \includegraphics[width=\linewidth]{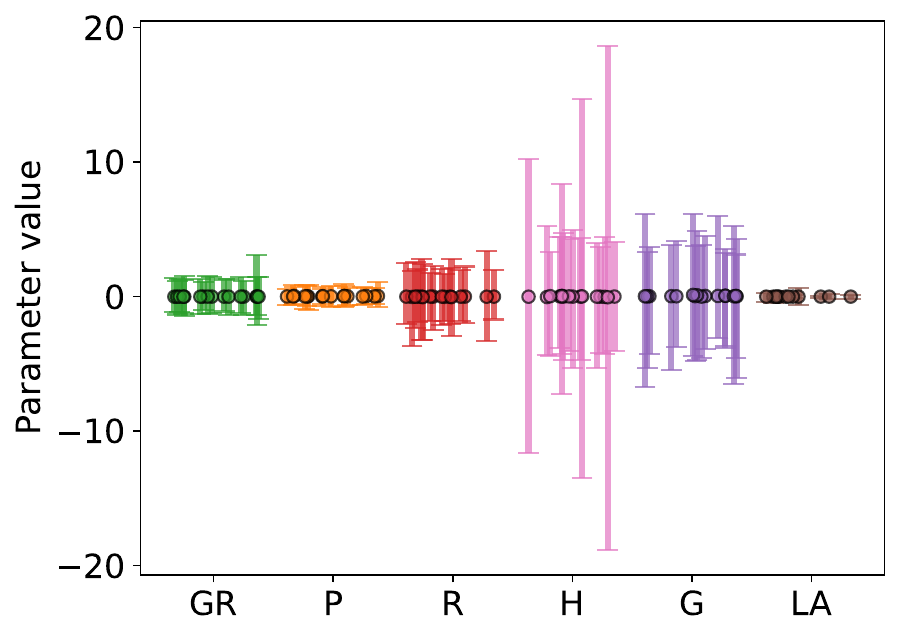}
        \caption{Dense linear Hamiltonian experiment from figure \ref{linear20}.}
	\end{subfigure}
    \begin{subfigure}{0.32\linewidth}
        \centering
        \includegraphics[width=\linewidth]{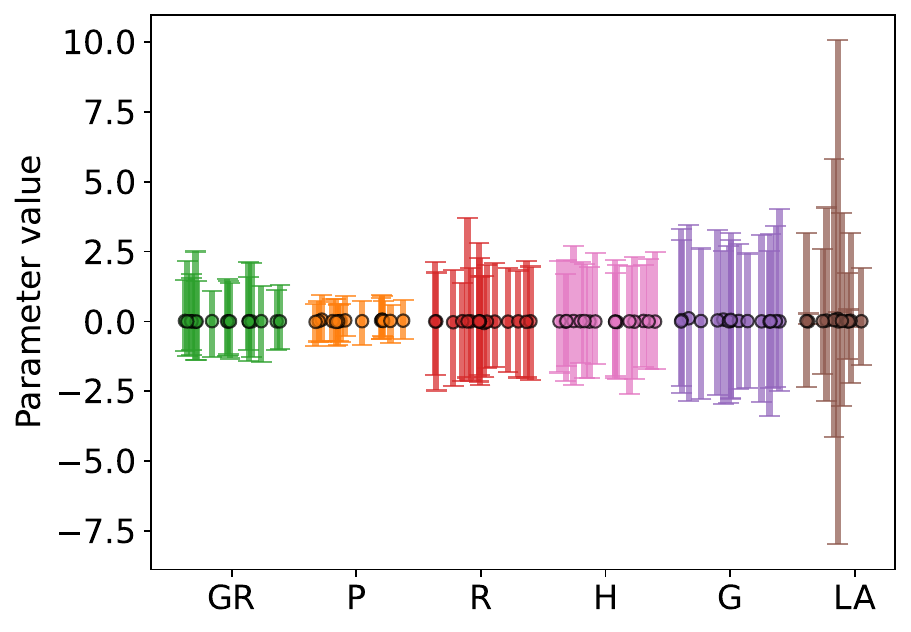} 
        \caption{Hénon-Heiles experiment for figure \ref{henon a}.}\label{henon params}
	\end{subfigure} 
	\vskip -0.1in
	\caption{The distribution of the learned parameters for the SympNets for various experiments. The dots are the average values and the error bars are the range that the maximum and minumum values take at the end of the training process. } \label{fig: params}
	\vskip -0.1in
\end{figure}

\section{Symbolic Hamiltonian regression with P-SympNets}\label{sec: symbolic regression}

In this section we will take advantage of property \ref{property 2 small basis} and calculate increasingly accurate approximations to the true Hamiltonian by expanding the inverse modified Hamiltonian in the BCH series. 

We denote by $\mathcal{B}^p_{H_i}$ the truncated $k=2$ backward error map defined by 
\begin{align}
\mathcal{B}^p_{\bar{H}_i}(\bar{H}_j) = & \bar{H}_i + \bar{H}_j + \frac{h}{2}\{\bar{H}_i, \bar{H}_j\} +\dfrac{h^2}{12}(\{\bar{H}_i, \{\bar{H}_i, \bar{H}_j\}\} + \{\bar{H}_j, \{\bar{H}_j, \bar{H}_i\}\}) + \dots + h^{p}\bar{H}^{[p]}, 
\end{align} 
which denotes the first $p$ terms of the modified Hamiltonian of the $k=2$ map. Note that when the basis Hamiltonians are polynomial of degree $d$, the degree of each term in the expansion increases like $O(p(d-1))$. For a $k>2$ map we can use induction to approximate the true Hamiltonian up to $O(h^p)$ by the following.
\begin{definition}[Truncated backward error map]
	Given some $p\ge1$ and  a map of the form \eqref{lie trotter map} such that it is the exact flow of a Hamiltonian system $\Phi_h^{\bar{H}}=\phi_h^{H}$ for the basis Hamiltonians $\{\bar{H}_i\}_{i=1}^k$, then the Hamiltonian $H$ can be approximated to $O(h^{p+1})$ by the truncated backward error map
	\begin{equation}\label{ibea}
	\mathcal{B}^p(\bar{H}_k,\dots,\bar{H}_1) := \mathcal{B}^p_{\bar{H}_k}\circ\dots\circ \mathcal{B}^p_{\bar{H}_1}(0)
	\end{equation}
	where $\mathcal{B}^p_{\bar{H}_k}(0)=\bar{H}_k$ and satisfies 
	\begin{equation}
	\|H - \mathcal{B}^p(\bar{H}_k,\dots,\bar{H}_1)\|\le  O(h^{p+1}).
	\end{equation}
\end{definition}

Note that in the realistic case where we find a non-exact approximation to $\bar{H}$ such that $\Phi_h^{\bar{H}} \approx \phi_h^{H}$ then we can still approximate the true Hamiltonian up to the level of accuracy that we have learned $\bar{H}$. In practice the derivatives in \eqref{ibea} can be computed using automatic differentiation if symbolic computations are slow. 

We note that if the map \eqref{lie trotter map} is smooth and close to the identity then by \citep[Lemma 5.3][]{hairer2006geometric} it is a solution to the Hamiltonian-Jacobi equations, and therefore the modified Hamiltonian $\bar{H}$ is guaranteed to exist. This is also a result of the fact that compositions of Hamiltonian flows are also Hamiltonian flows  \citep{polterovich2012geometry}. 

In this section we will reconstruct the exact Hamiltonians from the learned inverse modified Hamiltonians of P-SympNets. We will take P-SympNets trained on the polynomial Hamiltonians and apply the backward error map \eqref{ibea} to the learned basis Hamiltonians. 
 
Expressing the true Hamiltonian in its monomial basis up to degree $d$ we have
\begin{equation}
	H(x) = \sum_{|\mathbf{a}|\le d}^p \alpha_\mathbf{a} \mathbf{x}^{\mathbf{a}}\in\Pi_d^n
\end{equation}
where $p=\dim (\Pi_d^n)=\binom{n+d}{d}$ and $\mathbf{a}\in\mathbb{N}^n_+$ is a multi-index such that $\mathbf{x}^{\mathbf{a}} = x_1^{a_{1}}x_2^{a_{2}}\dots x_n^{a_{n}}$ and $|\mathbf{a}|=a_1 + ... + a_n$. The learned Hamiltonian can be expressed in the same basis 
\begin{equation}
	H^\theta(x) = \sum_{|\mathbf{a}|\le d}^p \alpha^\theta_\mathbf{a} \mathbf{x}^{\mathbf{a}}
\end{equation}
and the mean absolute error (MAE) for the coefficients of the learned polynomials are calculated by
\begin{equation}
	\overline{\Delta \alpha} = \frac{1}{n_{\text{nonzero}}} \sum_{\substack{i=1\\\alpha_i^{\theta}\ne 0}}^{n_{\text{nonzero}}}\left|\alpha^{\theta}_i - \alpha_i\right|
\end{equation}
where $n_{\text{nonzero}}$ are the number of non-zero coefficients in the true and learned Hamiltonians. 

The truncated backward error map is implemented in the \texttt{strupnet} Python package as a method in the \texttt{SympNet} class. 

\subsection{Double mass-spring system}
To illustrate the method, we train a P-SympNet on the following quadratic Hamiltonian, corresponding to a double mass-spring system with $n=2$ degrees of freedom 
\begin{equation} 
	H(x) = 0.5 x_0^2 + 0.5x_1^2 + 0.6 x_2^2 -0.4 x_2x_3 + 0.6x_3^2
\end{equation}
A degree $d=2$, $k=6$ layer P-SympNet is trained on a data set of 200 initial conditions, with timestep $h=0.02$ until it reaches a train and test set MSE loss of $O(10^{-18})$. Indicating it has learned the exact map up to a MAE of about $O(10^{-9})$. We then apply the backward error map corrections \eqref{ibea} up to $p=5$ to the learned Hamiltonian basis functions $\bar{H}^{\theta}_1,...,\bar{H}^{\theta}_k$. These are presented in Table \ref{double oscillator table} along with the MAE of the coefficients of the learned Hamiltonian. Indeed, we can recover the coefficients of the learned Hamiltonian to within about nine digits of precision.

\begin{table*}[h!]
	\centering
	\scriptsize
	\renewcommand{\arraystretch}{1.5}  % Adjust row spacing here 
	\begin{tabular}{c|c|c}
		Order & $\mathcal{B}^{ p }(\bar{H}^{\theta}_k,...,\bar{H}^{\theta}_1)$ & $\overline{\Delta \alpha}$\\
		\hline 
		$p=0$ & $0.5138323711 x_{0}^{2} + 0.4936377383 x_{1}^{2} + 0.5896792413 x_{2}^{2} - 0.3954104768 x_{2} x_{3} + 0.6007668497 x_{3}^{2} +O(\overline{\Delta \alpha})$ & 6.0e-3\\
		$p=1$ & $0.4997101493 x_{0}^{2} + 0.4997643881 x_{1}^{2} + 0.5998401567 x_{2}^{2} - 0.3999735820 x_{2} x_{3} + 0.6000121882 x_{3}^{2} +O(\overline{\Delta \alpha})$ & 1.6e-4\\
		$p=2$ & $0.5000077298 x_{0}^{2} + 0.4999952029 x_{1}^{2} + 0.5999965388 x_{2}^{2} - 0.4000019953 x_{2} x_{3} + 0.6000011669 x_{3}^{2} +O(\overline{\Delta \alpha})$ & 3.9e-6\\
		$p=3$ & $0.4999997882 x_{0}^{2} + 0.4999998107 x_{1}^{2} + 0.5999999318 x_{2}^{2} - 0.4000000419 x_{2} x_{3} + 0.5999999992 x_{3}^{2} +O(\overline{\Delta \alpha})$ & 1.1e-7\\
		$p=4$ & $0.5000000080 x_{0}^{2} + 0.4999999986 x_{1}^{2} + 0.5999999968 x_{2}^{2} - 0.3999999962 x_{2} x_{3} + 0.5999999998 x_{3}^{2} +O(\overline{\Delta \alpha})$ & 3.0e-9\\
		$p=5$ & $0.5000000023 x_{0}^{2} + 0.5000000026 x_{1}^{2} + 0.5999999981 x_{2}^{2} - 0.3999999941 x_{2} x_{3} + 0.5999999993 x_{3}^{2} +O(\overline{\Delta \alpha})$ & 2.3e-9	\\
		% PySINDY & $0.5000175211 x_{0}^{2} + 0.5000175183 x_{1}^{2} + 0.6000236697{2}^{2} - 0.4000283937 x_{2} x_{3} + 0.6000236715 x_{3}^{2} +O(\overline{\Delta \alpha})$ & 3.6e-5	\\
	\end{tabular}

	\caption{The corrected Hamiltonians for the double oscillator system with $H(x) = 0.5 x_0^2 + 0.5x_1^2 + 0.6 x_2^2 -0.4 x_2x_3 + 0.6x_3^2$. The term $O(\overline{\Delta \alpha})$ denotes monomials whose coefficients are of magnitude proportional to the numbers on the right-most column.}\label{double oscillator table}
\end{table*}

Note that due to the fact that the MAE loss is about $O(1e-9)$, which is also the accuracy of the $p=4$ backward error map approximation, increasing the order of the approximation to $p=5$ does not improve the accuracy. That is, the accuracy of the Hamiltonian cannot exceed that of the learned map (to within $O(h)$).

\subsection{Hénon-Heiles system}
We now take a P-SympNet trained on the Hénon-Heiles Hamiltonian presented in the previous section that achieved a test and train loss of about $O$(1e-12). We apply the backward error map up to degree $p=3$ to the learned basis Hamiltonians and present the learned corrections in table \ref{henon table}. In this case, we can recover the coefficients of the learned Hamiltonian to within about 5 digits of precision, which is about the accuracy that we have learned the exact map.
\begin{table*}[h!]
	\centering
	\scriptsize
	\renewcommand{\arraystretch}{1.5}  % Adjust row spacing here 
	\begin{tabular}{c|c|c}
		Order & $\mathcal{B}^{ p }(\bar{H}^{\theta}_k,...,\bar{H}^{\theta}_1)$ & $\overline{\Delta \alpha}$\\
		\hline 
	$p=0$ & $0.49404 x_{0}^{2} + 0.49368 x_{1}^{2} + 1.00437 x_{2}^{2} x_{3} + 0.51055 x_{2}^{2} + 0.99338 x_{3}^{3} + 0.49968 x_{3}^{2} +O(\overline{\Delta \alpha})$ & 1.6e-2  \\
	$p=1$ & $0.49965 x_{0}^{2} + 0.49940 x_{1}^{2} + 0.99817 x_{2}^{2} x_{3} + 0.49981 x_{2}^{2} + 0.99811 x_{3}^{3} + 0.49985 x_{3}^{2} +O(\overline{\Delta \alpha})$ & 7.1e-4  \\
	$p=2$ & $0.50003 x_{0}^{2} + 0.50003 x_{1}^{2} + 1.00009 x_{2}^{2} x_{3} + 0.50015 x_{2}^{2} + 1.00004 x_{3}^{3} + 0.50007 x_{3}^{2} +O(\overline{\Delta \alpha})$ & 9.7e-5  \\
	$p=3$ & $0.50004 x_{0}^{2} + 0.50004 x_{1}^{2} + 1.00011 x_{2}^{2} x_{3} + 0.50014 x_{2}^{2} + 1.00009 x_{3}^{3} + 0.50007 x_{3}^{2} +O(\overline{\Delta \alpha})$ & 2.9e-5\\
	\end{tabular}

	\caption{The corrected Hamiltonians for the Hénon-Heiles system $H = 0.5x_0^2+0.5x_1^2 + 0.5x_2^2+0.5x_3^2 + x_2^2 x_3 + x_3^3$. The term $O(\overline{\Delta \alpha})$ denotes monomials whose coefficients are of magnitude proportional to the numbers on the right-most column. }\label{henon table}
\end{table*}

\section{Discussion and conclusion}
One of the common observations throughout the experiments is that increasing the number of parameters, e.g., by increasing the layers and/or widths, does not necessarily increase the accuracy of the models. This can be seen by the fact that the models do not usually become more accurate as the training time increases. This is somewhat at odds with the fact that the models are all universal approximators for Hamiltonian flows. However, this suggests that there might be challenges with the training, and more specialized methods could more effectively find the optimal hyperparameters for the problem. We have suggested a possible explanation for this, namely the fact that models might require learning large parameter values. Hence, it is possible that one simply needs to train the models for a lot longer. A notable exception to this are the P-SympNets that often increases in accuracy with the training cost and number of parameters. This could be due to reasons explained in the following paragraph. 

In general, the P-SympNets consistently outperform the other models in terms of cost, accuracy, and generalizability. There are a number of potential reasons that could explain this. One, due to the fact that P-SympNets can represent linear Hamiltonian flows exactly, data sets that are dominated by linear dynamics and have perturbative or smaller nonlinear terms could be learned more effectively. Two, all the Hamiltonians used in the training data are polynomial or analytic functions, which are in a sense the ``smoothest" class of functions one could conceive. So approximating the inverse modified Hamiltonian with a polynomial is in a sense introducing a ``smoothness" inductive bias into the model that evidently works very well. It would be interesting to see whether P-SympNets perform equally well on Hamiltonians that are not analytic. However, in \cite{lin2017does}, the argument is made that many physical systems arise from low order polynomials, meaning that it could be wise to inform our neural models of this fact. It would therefore make sense to use P-SympNets in this context.  
 
When it comes to distinguishing between the G, H and R-SympNets, it is difficult to single out any general trends across all experiments. For separable systems (e.g., the Hénon-Heiles and Fermi-Pasta-Ulam systems) these methods are similar in terms training costs and accuracy. This could mean that models with a separable inverse modified Hamiltonian (G, LA and H-SympNets) are a sensible choice when the dynamics is generated by a separable Hamiltonian process. However, it can be seen that the GR-SympNet performs very well for the double pendulum experiment, which is likely due to its increased expressiveness compared to the G-SympNet and the fact that the true Hamiltonian is non-separable. However, this performance is nearly matched by the H-SympNet in many cases. 

A logical question to ask is how the models handle noise in the data. We remark that all the SympNets presented in this paper are amenable to methods such as initial state optimization \citep{chen2019symplectic} or mean inverse integration \citep{noren2023learning}, which have been shown to effectively reduce the effect of noise in the training process. We haven't considered noisy data sets in this paper as we believe it doesn't make sense to consider noisy data sets without adapting our methods to the above methods. However, it is a natural next step to adapt our methods to handle realistic systems, including dissipation and noise. 
 
One theoretical question that remains open is whether the generating Hamiltonian set for GR-SympNets (i.e., $\mathcal{H}$ from Definition \ref{def:grsympnet}) is dense in compact subsets of $C^1(\R^{2m})$. This would explain the fact that they learn reasonably small parameters as seen in Section \ref{sec: learned parameters} as well as explain their efficiency in approximating Hamiltonian flows and provide an alternate proof of their universality according to Theorem \ref{thm: universality}. 
 
\section*{Acknowledgements} 
The author would like to thank Zhen Zhang, Davide Murari and Sølve Eidnes for helpful discussions and comments. This work has received funding from the PRAI project (308832), and the Norwegian Research Council. 
%%%%%%%%%%%%%%%%%%%%%%%%%%%%%%%%%%%%%%%%%%%%%%%%%%%%%%%%%%%%%%%%%%%%%%%%%%%%%%%
%%%%%%%%%%%%%%%%%%%%%%%%%%%%%%%%%%%%%%%%%%%%%%%%%%%%%%%%%%%%%%%%%%%%%%%%%%%%%%%
% APPENDIX
%%%%%%%%%%%%%%%%%%%%%%%%%%%%%%%%%%%%%%%%%%%%%%%%%%%%%%%%%%%%%%%%%%%%%%%%%%%%%%%
%%%%%%%%%%%%%%%%%%%%%%%%%%%%%%%%%%%%%%%%%%%%%%%%%%%%%%%%%%%%%%%%%%%%%%%%%%%%%%%
\appendix

\section{Proof of Theorem \ref{thm: universality}}\label{sec:proof universality}
Our approach here is similar to universality results given in \cite{celledoni2023dynamical}. First we recall a lemma by Gronwall, which bounds the difference between two solutions of an ODE.
\begin{lemma}[Gronwall's inequality]\citep{howard1998gronwall}\label{lem: gronwall}
	Let ${X}$ be a Banach space and $U \subset {X}$ an open set. Let $f, g : [0, T] \times U \to {X}$ be continuous functions and let $y, z: [0, T] \to U$ satisfy
\begin{align}
\dot{y}(t) &= f(t, y(t)), \\
\dot{z}(t) &= g(t, z(t)),
\end{align}
for $ y(0) = y_0$ and $ z(0) = z_0$. Also assume there is a constant $C \ge 0$ so that
\begin{equation}
\|g(t, x_2) - g(t, x_1)\| \le C \|x_2 - x_1\| 
\end{equation}
{and a continuous function $\rho: [0, T] \to [0, \infty)$ so that}
\begin{equation}
\|f(t, y(t)) - g(t, y(t))\| \le \rho(t).
\end{equation}

Then for $t \in [0, T]$
\begin{equation}
\|y(t) - z(t)\| \le e^{C|t|} \|y_0 - z_0\| + e^{C|t|} \int_0^t e^{-C|s|} \rho(s) \, ds.
\end{equation}
\end{lemma}

We now proceed by showing that the exact flows on two nearby Hamiltonian systems is bounded. 
\begin{lemma}\label{lemma:flow bound}
	Let $\phi^{H_1}_h$ and $\phi^{H_2}_h$ be the flows of Hamiltonian systems with $H_1=H$ and $H_2=H+\epsilon H_0$ for some $\epsilon>0$. Furthermore, assume that the Hamiltonian vector field $ X_H $ is Lipschitz continuous on $\Omega$ with Lipschitz constant $ L $ and that $\|X_{H_0}\|<M$ is bounded on $\Omega$ by $M$. Then for any $x$ in some compact $\Omega\subset \R^{2n}$ and $h>0$ we have
	\begin{equation}
	\|\phi^{H_1}_h(x)-\phi^{H_2}_h(x)\| \leq  \epsilon M \frac{e^{L h} - 1}{L}.
	\end{equation}
\end{lemma} 
\begin{proof}
	By definition, the flows satisfy
	\begin{align}
		\frac{d}{dt} \phi^{H_1}_{t}(x) &= X_H(\phi^{H_1}_{t}(x))\\
		\frac{d}{dt} \phi^{H_2}_{t}(x) &= X_{H+\epsilon H_0}(\phi^{H_2}_{t}(x)) =X_{H}(\phi^{H_2}_{t}(x)) + \epsilon X_{H_0}(\phi^{H_2}_{t}(x)) 
	\end{align}
	using the fact that $X_H$ is a linear operator. Applying Gronwall's inequality from Lemma \ref{lem: gronwall} we obtain
	\begin{equation}
	\| \phi^{H_1}_h(x) - \phi^{H_2}_h(x) \| \leq e^{L|t|} \int_0^t e^{-L|s|} \epsilon M \, ds \leq \epsilon M \frac{e^{L h} - 1}{L}.
	\end{equation}

\end{proof}

We now state a standard result in numerical analysis of the convergence of the global error of a one-step method.
\begin{lemma}\citep{hairer1987solving}\label{lemma: global error}
	Let $\Phi_t$ denote an order-one numerical approximation to a smooth flow map $\phi_t$
	\begin{equation}
	\|\underbrace{{\Phi}_{t/m}\circ\dots\circ{\Phi}_{t/m}}_{m\text{ times}}(x) - \phi_{t}(x)\| \leq \frac{tK}{m}(e^{tL}-1),
	\end{equation}
	for some constants $K$ and $L$.
\end{lemma}

The main point here is that $m$ can be taken to be arbitrarily large, and the global error can be made arbitrarily small. We can now prove the main theorem.

\begin{proof}\textbf{of Theorem \ref{thm: universality}.}
	As the theorem assumes that $\mathrm{span}\{H^{\theta}(x): \theta\in\Theta\}$ is dense in $C^1(\Omega)$ for some compact set $\Omega$, there must exist $l$ functions $\{H^\theta_{i}\}_{i=1}^l$ such that for any $\epsilon>0$
	$$\|\sum_{i=1}^{l}H^\theta_{i}(x) - H(x)\|_{C_1(\Omega)} < \epsilon, \quad \forall x\in \Omega.$$
	We can therefore write  $\bar{H}^\theta(x) = \sum_{i=1}^{l}H^\theta_{i}(x)$ as $H$ plus an error term that is bounded in the $C_1(\Omega)$ norm 
	$$\bar{H}^\theta(x) = H(x) + \epsilon \Delta H(x), \quad\text{where } \|\Delta H(x)\|_{C_1(\Omega)}<1\quad, \forall x\in \Omega.$$
	Using Lemma \ref{lemma:flow bound} combined with the above, the exact flow on $X_{\bar{H}^{\theta}}$ can be made arbitrarily close to the exact flow on $X_{H}$ 
	$$\|\phi^{\bar{H}^{\theta}}_t(x)-\phi_t^{H}(x)\| < \frac{\epsilon}{2}$$ 
	for $t\in[0,T]$ and $x\in\Omega$. 

	Now let $\hat{\Phi}_{h}^{\bar{H}^{\theta}}=\phi^{H_{\theta_i}}_h\circ\dots\circ\phi^{H_{\theta_l}}_h$, which is an order one approximation to $\phi^{\bar{H}_\theta}_h$ satisfying $\|\hat{\Phi}_{h}^{\bar{H}^{\theta}}-\phi^{\bar{H}_\theta}_h\|<Ch^2$, for sufficiently small $h$ and some constant $C$. Then using Lemma \ref{lemma: global error}, there must exist some integer $m$ such that $m$ compositions of $\hat{\Phi}_{t/m}^{\bar{H}^{\theta}}=\hat{\Phi}_{t}^{\frac{1}{m}\bar{H}^{\theta}}$ can be made arbitrarily close to the exact flow on $X_{\bar{H}^{\theta}}$, that is 
	\begin{equation}
		\|\underbrace{\hat{\Phi}_{t/m}^{\bar{H}^{\theta}}\circ\dots\circ\hat{\Phi}_{t/m}^{\bar{H}^{\theta}}}_{m\text{ times}}(x) - \phi^{\bar{H}^\theta}_t(x) \| < \frac{\epsilon}{2}.
	\end{equation}
	By denoting the $k=lm$ layer map by $\Phi_{t}^{\bar{H}^{\theta}}=\hat{\Phi}_{t/m}^{\bar{H}^{\theta}}\circ\dots\circ\hat{\Phi}_{t/m}^{\bar{H}^{\theta}}=\hat{\Phi}_{t}^{\frac{1}{m}\bar{H}^{\theta}}\circ\dots\circ\hat{\Phi}_{t}^{\frac{1}{m}\bar{H}^{\theta}}$, and combining the above two inequalities, we have as required 
	\begin{equation}
		\|\Phi_{t}^{\bar{H}^{\theta}}(x)-\phi^{H}_t(x)\| < \epsilon, \quad \forall x\in \Omega.
	\end{equation}
\end{proof}

\section{Proof of Proposition \ref{thm:GR basis}}\label{sec:proof of gr thm}
We make use of the following.
\begin{lemma}{\citep{jin2022optimal}}\label{thm: triangular amtrix factorisation}
	Assume $A$ is invertible. Let the following denote a symplectic matrix $$M= \begin{pmatrix}
		A & B \\
		C & D \\
	\end{pmatrix}
	\in\R^{2n\times2n}, \quad M J M^T = J,$$ 
	for $A, B, C, D\in\R^{n\times n}$ and denote a triangular matrix by $$T_i = \begin{pmatrix}
		I &  S_i\\
		0 & I \\
	\end{pmatrix}
	$$ where $S_i=S_i^T\in\R^{n\times n}$ are symmetric. Then there exist some $S_1$, $S_2$, $S_3$ and $S_4$ such that any $M\in Sp(n, \R)$ can be factored into  $$M=T_4^TT_3T_2^TT_1.$$ 
\end{lemma}
\begin{proof}\textbf{of Proposition \ref{thm:GR basis}}
	Let $M$ be a block symplectic matrix as in Lemma \ref*{thm: triangular amtrix factorisation}. Then $MJM^T=J$ implies $AB^T=BA^T$. Expressing these matrices in terms of $S_i$ yields
\begin{align}\label{tsympnet basis}
\begin{split}
A &= I + S_3S_2,\\
B &= S_3 + S_3 S_2 S_1 + S_1.
% C &= S_4 + S_4S_3S_2+S_2= S_4A+S_2,\\
% D &= 1 + S_2S_1 + S_4S_1 + S_4S_3 + S_4S_3S_2S_1,
\end{split}
\end{align}
% Due to Lemma \ref{thm: triangular amtrix factorisation} choosing $A$ and $B$ of the form \eqref{tsympnet basis} will satisfy $AB^T=BA^T$. 
\end{proof}

\section{Proof of Theorem \ref{thm: representation of linear hamiltonian flows}}\label{app:proof of linear hamiltonian flows}

% Define the following Hamiltonians 
% \begin{equation}
% 	L_i(p;a) = a p_i^2, \quad L_{ij}(p;a) = a p_ip_j \quad R_i(q;a) = a q_i^2, \quad R_{ij}(q;a) = a q_iq_j, 
% \end{equation}
% where the $a\in\R$ are free parameters. 

% \begin{lemma}
% 	The symplectic group $Sp(U)$ is generated by the family of Hamiltonian shear maps:
% 	\begin{equation}
% 		\{\phi^{L_i}_h, \phi^{L_{ij}}_h, \phi^{R_i}_h, \phi^{R_{ij}}_h\}=\mathcal{S}
% 	\end{equation}
% 	i.e., if $M\in Sp(U)$, then $M=\prod_{i=1}^{k} g_i$ for $g\in\mathcal{S}$.
% \end{lemma}
% \begin{proof}
% 	The proof is a direct consequence of theorem 2.1 and lemma 3.2 in \cite{janeczko2009characterization}.
% \end{proof}
We begin by stating a lemma from \citep{jin2022optimal} that gives the optimal factorization of a symplectic matrix into unit triangular matrices. 
\begin{lemma}\citep{jin2022optimal}\label{optimal factorisation}
For any symplectic matrix $M \in \R^{2n \times 2n}$, there exist $\delta_1, \cdots, \delta_n \in \{0, 1\}$ and symmetric $S_i=S_i^T \in \R^{d \times d}$ such that
	\begin{equation}
	M = \begin{pmatrix}
	I & \text{diag}(\delta_1, \cdots, \delta_d) \\
	0 & I 
	\end{pmatrix}
	\begin{pmatrix}
	I & 0 \\
	S_4 & I 
	\end{pmatrix}
	\begin{pmatrix}
	I & S_3 \\
	0 & I 
	\end{pmatrix}
	\begin{pmatrix}
	I & 0 \\
	S_2 & I 
	\end{pmatrix}
	\begin{pmatrix}
	I & S_1 \\
	0 & I 
	\end{pmatrix}
	\end{equation}
\end{lemma}

We now show that any unit triangular matrix transformation can be represented by a P-SympNet. 

\begin{lemma}\label{triangular psympnet}
	Any triangular matrix transformation of the form $Tx$ or $T^Tx$ where 
	\begin{equation}
		T = \begin{pmatrix}
			I & S \\
			0 & I 
		\end{pmatrix}
	\end{equation} 
	for some symmetric $S=S^T$ can be represented by at most $n$ P-SympNet layers. 
\end{lemma}
\begin{proof} Write $S$ in its spectral decomposition $S= \sum_{i=1}^n w_i w_i^T$, where $w_i=\lambda_i v_i$ for non-negative eigenvalues $\lambda_i^2$ and orthogonal eigenvectors $v_i$ of $S$. This yields  
	\begin{equation}
		\begin{pmatrix}
			I & \sum_{i=1}^n w_iw_i^T \\
			0 & I 
		\end{pmatrix}
		\begin{pmatrix}
			p\\q\\
		\end{pmatrix} 
		=\prod_{i=1}^n\begin{pmatrix}
			I & w_iw_i^T \\
			0 & I 
		\end{pmatrix}
		\begin{pmatrix}
			p\\q\\
		\end{pmatrix} 
		 = \phi^{(w_1^Tq)^2}_1\circ...\circ\phi^{(w_n^Tq)^2}_1 (x)
	\end{equation}
	where $\phi^{(w_i^Tq)^2}_1 = e^{J\hat{w}_i\hat{w}_i^T}=(I + J\hat{w}_i\hat{w}_i^T)$ is the $i$th layer of a P-SympNet with quadratic ridge functions and direction vector $\hat{w}_i = (0, w_i)\in\R^{2n}$, $w_i\in\R^n$, and timestep $h=1$. The $h>0$ case can be achieved by scaling $w_i$. The same argument holds for transformations of the type $T^Tx$ using composition of $\phi^{(w_i^Tp)^2}_1$.
\end{proof}

Lastly, we state an elementary result from \citep{lee2012smooth}. 

\begin{lemma}\citep[Proposition 20.9,][]{lee2012smooth}\label{lemma: lee}
	Let $G$ be a Lie group and let $\mathfrak{g}$ be its Lie algebra. For any $X, Y \in \mathfrak{g}$, there is a smooth function $Z : (-\epsilon, \epsilon) \to \mathfrak{g}$ satisfying $Z(0) = 0$, and such that the following identity holds for all $t \in (-\epsilon, \epsilon)$: 
	$$\exp(tX)\exp(tY) = \exp(t(X + Y + Z(t))).$$
\end{lemma}
Now we can prove the theorem.
\begin{proof}\textbf{of Theorem \ref{thm: representation of linear hamiltonian flows}.}
We first note that any P-SympNet of degree $d$ is a super set of P-SympNets of degree two, hence it suffices to to prove the theorem for P-SympNets of degree two. We consider the three cases separately.  

Case $(i)$: Using Lemma \ref{optimal factorisation} we can write any symplectic matrix in $2n$ dimensions as a product of 5 unit triangular matrices. Therefore, using Lemma \ref{triangular psympnet}, we can write such a matrix as a product of at most $5n$ P-SympNet layers.

Case $(ii)$: If we further assume that $A$ is invertible as in Lemma \ref{thm: triangular amtrix factorisation}, we can write a symplectic matrix as a product of only four unit triangular matrices. This can therefore be represented by a P-SympNet with at most $4n$ layers using Lemma \ref{triangular psympnet}. 

Case $(iii)$: Let the set $\{y_i\}_{i=1}^{2n}$, where $y_i=Jw_iw_i^T$, denote a basis for the algebra $\mathfrak{g}$ of Hamiltonian matrices of the form $JM$. Using Lemma \ref{lemma: lee} we have
$$e^{ha_1y_1}\dots e^{ha_{2n}y_{2n}} = e^{h(a_1y_1 + ... +  a_{2n}y_{2n}) + hZ(h)}$$
   
As $Z(h)\in\mathfrak{g}$ is also in the Lie algebra, and $h$ is fixed, we can choose $a_1y_1 + ... +  a_{2n}y_{2n} = (b_1y_1 + ... +  b_{2n}y_{2n}) - Z(h)$. Noting that a quadratic P-SympNet layer is of the form $$\phi^{\frac{1}{2}(w_i^Tx)^2}_h(x) = (I + hJw_iw_i^T)x = e^{hJw_iw_i^T}x,$$ due to the fact that $Jw_iw_i^T$ is nilpotent. Then
$$\Phi_h^{\bar{H}^\theta}(x)=e^{ha_1y_1}\dots e^{ha_{2n}y_{2n}}x  =  e^{h(b_1y_1 + ... +  b_{2n}y_{2n})}x = e^{hJM}x$$
for any $JM\in\mathfrak{g}$, as required.
\end{proof}

\section{Proof of Theorem \ref{thm:grsympnet density}} \label{proof of grsympnet density}
We begin by restating the universal approximation theorem for G-SympNets.
\begin{lemma}\citep{jin2020sympnets}
	For any $r>0$ and open $\Omega \subset R^{2n}$, the set of G-SympNets is r-uniformly dense on a compact subset of $Sp^r(\Omega)$ if the activation function $\sigma$ is $r$-finite.
\end{lemma}
\begin{proof}\textbf{of Theorem \ref{thm:grsympnet density}}
The set of GR-SympNets is a superset of G-SympNets, and are equal when $A=I$ and $B=0$ or $A=0$ and $B=I$. Therefore, the result follows as a corollary of the above lemma.
\end{proof}

\bibliography{bibliography}

\begin{thebibliography}{55}
\providecommand{\natexlab}[1]{#1}
\providecommand{\url}[1]{\texttt{#1}}
\expandafter\ifx\csname urlstyle\endcsname\relax
  \providecommand{\doi}[1]{doi: #1}\else
  \providecommand{\doi}{doi: \begingroup \urlstyle{rm}\Url}\fi

\bibitem[Baj{\=a}rs(2023)]{bajars2023locally}
J{\=a}nis Baj{\=a}rs.
\newblock Locally-symplectic neural networks for learning volume-preserving
  dynamics.
\newblock \emph{Journal of Computational Physics}, 476:\penalty0 111911, 2023.

\bibitem[Berger and Turaev(2022)]{berger2022generators}
Pierre Berger and Dmitry Turaev.
\newblock Generators of groups of hamitonian maps.
\newblock \emph{arXiv preprint arXiv:2210.14710}, 2022.

\bibitem[Biagi et~al.(2020)Biagi, Bonfiglioli, and Matone]{biagi2020baker}
Stefano Biagi, Andrea Bonfiglioli, and Marco Matone.
\newblock On the baker-campbell-hausdorff theorem: non-convergence and
  prolongation issues.
\newblock \emph{Linear and Multilinear Algebra}, 68\penalty0 (7):\penalty0
  1310--1328, 2020.

\bibitem[Burby et~al.(2020)Burby, Tang, and Maulik]{burby2020fast}
Joshua~William Burby, Qi~Tang, and R~Maulik.
\newblock Fast neural poincar{\'e} maps for toroidal magnetic fields.
\newblock \emph{Plasma Physics and Controlled Fusion}, 63\penalty0
  (2):\penalty0 024001, 2020.

\bibitem[Celledoni et~al.(2019)Celledoni, Evripidou, McLaren, Owren, Quispel,
  Tapley, and van~der Kamp]{celledoni2019using}
Elena Celledoni, Charalambos Evripidou, David~I McLaren, Brynjulf Owren, Gilles
  Reinout~Willem Quispel, BK~Tapley, and Peter~H van~der Kamp.
\newblock Using discrete darboux polynomials to detect and determine preserved
  measures and integrals of rational maps.
\newblock \emph{Journal of Physics A: Mathematical and Theoretical},
  52\penalty0 (31):\penalty0 31LT01, 2019.

\bibitem[Celledoni et~al.(2023{\natexlab{a}})Celledoni, Leone, Murari, and
  Owren]{celledoni2023learning}
Elena Celledoni, Andrea Leone, Davide Murari, and Brynjulf Owren.
\newblock Learning hamiltonians of constrained mechanical systems.
\newblock \emph{Journal of Computational and Applied Mathematics},
  417:\penalty0 114608, 2023{\natexlab{a}}.

\bibitem[Celledoni et~al.(2023{\natexlab{b}})Celledoni, Murari, Owren,
  Sch{\"o}nlieb, and Sherry]{celledoni2023dynamical}
Elena Celledoni, Davide Murari, Brynjulf Owren, Carola-Bibiane Sch{\"o}nlieb,
  and Ferdia Sherry.
\newblock Dynamical systems--based neural networks.
\newblock \emph{SIAM Journal on Scientific Computing}, 45\penalty0
  (6):\penalty0 A3071--A3094, 2023{\natexlab{b}}.

\bibitem[Chang et~al.(2019)Chang, Chen, Haber, and
  Chi]{chang2019antisymmetricrnn}
Bo~Chang, Minmin Chen, Eldad Haber, and Ed~H Chi.
\newblock Antisymmetricrnn: A dynamical system view on recurrent neural
  networks.
\newblock \emph{arXiv preprint arXiv:1902.09689}, 2019.

\bibitem[Chen and Tao(2021)]{chen2021data}
Renyi Chen and Molei Tao.
\newblock Data-driven prediction of general hamiltonian dynamics via learning
  exactly-symplectic maps.
\newblock In \emph{International Conference on Machine Learning}, pages
  1717--1727. PMLR, 2021.

\bibitem[Chen et~al.(2018)Chen, Rubanova, Bettencourt, and
  Duvenaud]{chen2018neural}
Ricky~TQ Chen, Yulia Rubanova, Jesse Bettencourt, and David~K Duvenaud.
\newblock Neural ordinary differential equations.
\newblock \emph{Advances in neural information processing systems}, 31, 2018.

\bibitem[Chen et~al.(2023)Chen, Xu, Matsubara, and
  Yaguchi]{chen2023variational}
Yuhan Chen, Baige Xu, Takashi Matsubara, and Takaharu Yaguchi.
\newblock Variational principle and variational integrators for neural
  symplectic forms.
\newblock In \emph{ICML Workshop on New Frontiers in Learning, Control, and
  Dynamical Systems}, 2023.
\newblock URL \url{https://openreview.net/forum?id=XvbJqbW3rf}.

\bibitem[Chen et~al.(2019)Chen, Zhang, Arjovsky, and
  Bottou]{chen2019symplectic}
Zhengdao Chen, Jianyu Zhang, Martin Arjovsky, and L{\'e}on Bottou.
\newblock Symplectic recurrent neural networks.
\newblock \emph{arXiv preprint arXiv:1909.13334}, 2019.

\bibitem[David and M{\'e}hats(2023)]{david2023symplectic}
Marco David and Florian M{\'e}hats.
\newblock Symplectic learning for hamiltonian neural networks.
\newblock \emph{Journal of Computational Physics}, 494:\penalty0 112495, 2023.

\bibitem[E.~Hairer(1987)]{hairer1987solving}
G.~Wanner E.~Hairer, S.P.~Nørsett.
\newblock \emph{Solving Ordinary Differential Equations I}.
\newblock Springer, 1987.

\bibitem[Eidnes et~al.(2023)Eidnes, Stasik, Sterud, B{\o}hn, and
  Riemer-S{\o}rensen]{eidnes2023pseudo}
S{\o}lve Eidnes, Alexander~J Stasik, Camilla Sterud, Eivind B{\o}hn, and Signe
  Riemer-S{\o}rensen.
\newblock Pseudo-hamiltonian neural networks with state-dependent external
  forces.
\newblock \emph{Physica D: Nonlinear Phenomena}, 446:\penalty0 133673, 2023.

\bibitem[Feng and Wang(1998)]{feng1998variations}
Kang Feng and Dao-liu Wang.
\newblock Variations on a theme by euler.
\newblock \emph{Journal of Computational Mathematics}, pages 97--106, 1998.

\bibitem[Galimberti et~al.(2023)Galimberti, Furieri, Xu, and
  Ferrari-Trecate]{galimberti2023hamiltonian}
Clara~Luc{\'\i}a Galimberti, Luca Furieri, Liang Xu, and Giancarlo
  Ferrari-Trecate.
\newblock Hamiltonian deep neural networks guaranteeing nonvanishing gradients
  by design.
\newblock \emph{IEEE Transactions on Automatic Control}, 68\penalty0
  (5):\penalty0 3155--3162, 2023.

\bibitem[Haber and Ruthotto(2017)]{haber2017stable}
Eldad Haber and Lars Ruthotto.
\newblock Stable architectures for deep neural networks.
\newblock \emph{Inverse problems}, 34\penalty0 (1):\penalty0 014004, 2017.

\bibitem[Hairer et~al.(2006)Hairer, Hochbruck, Iserles, and
  Lubich]{hairer2006geometric}
Ernst Hairer, Marlis Hochbruck, Arieh Iserles, and Christian Lubich.
\newblock Geometric numerical integration.
\newblock \emph{Oberwolfach Reports}, 3\penalty0 (1):\penalty0 805--882, 2006.

\bibitem[Horn et~al.(2023)Horn, Saz~Ulibarrena, Koren, and
  Portegies~Zwart]{horn4555181generalized}
Philipp Horn, Veronica Saz~Ulibarrena, Barry Koren, and Simon Portegies~Zwart.
\newblock A generalized framework of neural networks for hamiltonian systems.
\newblock \emph{Available at SSRN 4555181}, 2023.

\bibitem[Hornik(1991)]{hornik1991approximation}
Kurt Hornik.
\newblock Approximation capabilities of multilayer feedforward networks.
\newblock \emph{Neural networks}, 4\penalty0 (2):\penalty0 251--257, 1991.

\bibitem[Howard(1998)]{howard1998gronwall}
Ralph Howard.
\newblock The gronwall inequality.
\newblock \emph{lecture notes}, 1998.

\bibitem[Ismailov(2020)]{ismailov2020notes}
Vugar Ismailov.
\newblock Notes on ridge functions and neural networks.
\newblock \emph{arXiv preprint arXiv:2005.14125}, 2020.

\bibitem[Jin et~al.(2020)Jin, Zhang, Zhu, Tang, and
  Karniadakis]{jin2020sympnets}
Pengzhan Jin, Zhen Zhang, Aiqing Zhu, Yifa Tang, and George~Em Karniadakis.
\newblock Sympnets: Intrinsic structure-preserving symplectic networks for
  identifying hamiltonian systems.
\newblock \emph{Neural Networks}, 132:\penalty0 166--179, 2020.
\newblock ISSN 0893-6080.
\newblock \doi{https://doi.org/10.1016/j.neunet.2020.08.017}.
\newblock URL
  \url{https://www.sciencedirect.com/science/article/pii/S0893608020303063}.

\bibitem[Jin et~al.(2022{\natexlab{a}})Jin, Lin, and Xiao]{jin2022optimal}
Pengzhan Jin, Zhangli Lin, and Bo~Xiao.
\newblock Optimal unit triangular factorization of symplectic matrices.
\newblock \emph{Linear Algebra and its Applications}, 650:\penalty0 236--247,
  2022{\natexlab{a}}.

\bibitem[Jin et~al.(2022{\natexlab{b}})Jin, Zhang, Kevrekidis, and
  Karniadakis]{jin2022learning}
Pengzhan Jin, Zhen Zhang, Ioannis~G Kevrekidis, and George~Em Karniadakis.
\newblock Learning poisson systems and trajectories of autonomous systems via
  poisson neural networks.
\newblock \emph{IEEE Transactions on Neural Networks and Learning Systems},
  2022{\natexlab{b}}.

\bibitem[Kang and Zai-Jiu(1995)]{kang1995volume}
Feng Kang and Shang Zai-Jiu.
\newblock Volume-preserving algorithms for source-free dynamical systems.
\newblock \emph{Numerische Mathematik}, 71:\penalty0 451--463, 1995.

\bibitem[Koch and Lomelí(2014)]{koch2014hamiltonian}
Hans Koch and Héctor~E. Lomelí.
\newblock On hamiltonian flows whose orbits are straight lines.
\newblock \emph{Discrete and Continuous Dynamical Systems}, 34\penalty0
  (5):\penalty0 2091--2104, 2014.
\newblock ISSN 1078-0947.
\newblock \doi{10.3934/dcds.2014.34.2091}.

\bibitem[Lee(2012)]{lee2012smooth}
John~M Lee.
\newblock \emph{Smooth manifolds}.
\newblock Springer, 2012.

\bibitem[Lin et~al.(2017)Lin, Tegmark, and Rolnick]{lin2017does}
Henry~W Lin, Max Tegmark, and David Rolnick.
\newblock Why does deep and cheap learning work so well?
\newblock \emph{Journal of Statistical Physics}, 168:\penalty0 1223--1247,
  2017.

\bibitem[Maslovskaya and Ober-Bl{\"o}baum(2024)]{maslovskaya2024symplectic}
Sofya Maslovskaya and Sina Ober-Bl{\"o}baum.
\newblock Symplectic methods in deep learning.
\newblock \emph{arXiv preprint arXiv:2406.04104}, 2024.

\bibitem[McLachlan and Quispel(2002)]{mclachlan2002splitting}
Robert~I McLachlan and G~Reinout~W Quispel.
\newblock Splitting methods.
\newblock \emph{Acta Numerica}, 11:\penalty0 341--434, 2002.

\bibitem[McLachlan and Quispel(2004)]{mclachlan2004explicit}
Robert~I McLachlan and G~Reinout~W Quispel.
\newblock Explicit geometric integration of polynomial vector fields.
\newblock \emph{BIT Numerical Mathematics}, 44\penalty0 (3):\penalty0 515--538,
  2004.

\bibitem[Meng et~al.(2022)Meng, Zhang, Darbon, and
  Karniadakis]{meng2022sympocnet}
Tingwei Meng, Zhen Zhang, Jerome Darbon, and George Karniadakis.
\newblock Sympocnet: Solving optimal control problems with applications to
  high-dimensional multiagent path planning problems.
\newblock \emph{SIAM Journal on Scientific Computing}, 44\penalty0
  (6):\penalty0 B1341--B1368, 2022.

\bibitem[Noren et~al.(2023)Noren, Eidnes, and Celledoni]{noren2023learning}
H{\aa}kon Noren, S{\o}lve Eidnes, and Elena Celledoni.
\newblock Learning dynamical systems from noisy data with inverse-explicit
  integrators.
\newblock \emph{arXiv preprint arXiv:2306.03548}, 2023.

\bibitem[Offen and Ober-Bl{\"o}baum(2022)]{offen2022symplectic}
Christian Offen and Sina Ober-Bl{\"o}baum.
\newblock Symplectic integration of learned hamiltonian systems.
\newblock \emph{Chaos: An Interdisciplinary Journal of Nonlinear Science},
  32\penalty0 (1), 2022.

\bibitem[Owren and Marthinsen(2001)]{owren2001integration}
Brynjulf Owren and Arne Marthinsen.
\newblock Integration methods based on canonical coordinates of the second
  kind.
\newblock \emph{Numerische Mathematik}, 87:\penalty0 763--790, 2001.

\bibitem[Pinkus(2015)]{pinkus2015ridge}
Allan Pinkus.
\newblock \emph{Ridge Functions}, volume 205.
\newblock Cambridge University Press, 2015.

\bibitem[Polterovich(2012)]{polterovich2012geometry}
Leonid Polterovich.
\newblock \emph{The geometry of the group of symplectic diffeomorphism}.
\newblock Birkh{\"a}user, 2012.

\bibitem[Raissi et~al.(2019)Raissi, Perdikaris, and
  Karniadakis]{raissi2019physics}
Maziar Raissi, Paris Perdikaris, and George~E Karniadakis.
\newblock Physics-informed neural networks: A deep learning framework for
  solving forward and inverse problems involving nonlinear partial differential
  equations.
\newblock \emph{Journal of Computational physics}, 378:\penalty0 686--707,
  2019.

\bibitem[Rezende et~al.(2019)Rezende, Racani{\`e}re, Higgins, and
  Toth]{rezende2019equivariant}
Danilo~Jimenez Rezende, S{\'e}bastien Racani{\`e}re, Irina Higgins, and Peter
  Toth.
\newblock Equivariant hamiltonian flows.
\newblock \emph{arXiv preprint arXiv:1909.13739}, 2019.

\bibitem[Sherry et~al.(2024)Sherry, Celledoni, Ehrhardt, Murari, Owren, and
  Sch{\"o}nlieb]{sherry2024designing}
Ferdia Sherry, Elena Celledoni, Matthias~J Ehrhardt, Davide Murari, Brynjulf
  Owren, and Carola-Bibiane Sch{\"o}nlieb.
\newblock Designing stable neural networks using convex analysis and odes.
\newblock \emph{Physica D: Nonlinear Phenomena}, 463:\penalty0 134159, 2024.

\bibitem[Shin and Ghosh(1995)]{shin1995ridge}
Yoan Shin and Joydeep Ghosh.
\newblock Ridge polynomial networks.
\newblock \emph{IEEE Transactions on neural networks}, 6\penalty0 (3):\penalty0
  610--622, 1995.

\bibitem[Tapley et~al.(2019)Tapley, Celledoni, Owren, and
  Andersson]{tapley2019novel}
Benjamin Tapley, Elena Celledoni, Brynjulf Owren, and Helge~I Andersson.
\newblock A novel approach to rigid spheroid models in viscous flows using
  operator splitting methods.
\newblock \emph{Numerical Algorithms}, 81\penalty0 (4):\penalty0 1423--1441,
  2019.

\bibitem[Tapley(2022)]{tapley2022geometric}
Benjamin~K Tapley.
\newblock Geometric integration of odes using multiple quadratic auxiliary
  variables.
\newblock \emph{SIAM Journal on Scientific Computing}, 44\penalty0
  (4):\penalty0 A2651--A2668, 2022.

\bibitem[Tapley(2023)]{tapley2023preservation}
Benjamin~K. Tapley.
\newblock On the preservation of second integrals by runge-kutta methods.
\newblock \emph{Journal of Computational Dynamics}, 10\penalty0 (2):\penalty0
  304--322, 2023.
\newblock ISSN 2158-2491.
\newblock \doi{10.3934/jcd.2023001}.
\newblock URL
  \url{https://www.aimsciences.org/article/id/6411474555cee464a30e0e8c}.

\bibitem[Tapley et~al.(2022)Tapley, Andersson, Celledoni, and
  Owren]{tapley2022computational}
Benjamin~K Tapley, Helge~I Andersson, Elena Celledoni, and Brynjulf Owren.
\newblock Computational geometric methods for preferential clustering of
  particle suspensions.
\newblock \emph{Journal of Computational Physics}, 448:\penalty0 110725, 2022.

\bibitem[Tong et~al.(2021)Tong, Xiong, He, Pan, and Zhu]{tong2021symplectic}
Yunjin Tong, Shiying Xiong, Xingzhe He, Guanghan Pan, and Bo~Zhu.
\newblock Symplectic neural networks in taylor series form for hamiltonian
  systems.
\newblock \emph{Journal of Computational Physics}, 437:\penalty0 110325, 2021.

\bibitem[Toth et~al.(2019)Toth, Rezende, Jaegle, Racani{\`e}re, Botev, and
  Higgins]{toth2019hamiltonian}
Peter Toth, Danilo~Jimenez Rezende, Andrew Jaegle, S{\'e}bastien Racani{\`e}re,
  Aleksandar Botev, and Irina Higgins.
\newblock Hamiltonian generative networks.
\newblock \emph{arXiv preprint arXiv:1909.13789}, 2019.

\bibitem[Treves(2016)]{treves2016topological}
Fran{\c{c}}ois Treves.
\newblock \emph{Topological Vector Spaces, Distributions and Kernels: Pure and
  Applied Mathematics, Vol. 25}, volume~25.
\newblock Elsevier, 2016.

\bibitem[Turaev(2002)]{turaev2002polynomial}
Dmitry Turaev.
\newblock Polynomial approximations of symplectic dynamics and richness of
  chaos in non-hyperbolic area-preserving maps.
\newblock \emph{Nonlinearity}, 16\penalty0 (1):\penalty0 123, 2002.

\bibitem[Valperga et~al.(2022)Valperga, Webster, Turaev, Klein, and
  Lamb]{valperga2022learning}
Riccardo Valperga, Kevin Webster, Dmitry Turaev, Victoria Klein, and Jeroen
  Lamb.
\newblock Learning reversible symplectic dynamics.
\newblock In \emph{Learning for Dynamics and Control Conference}, pages
  906--916. PMLR, 2022.

\bibitem[Xiong et~al.(2020)Xiong, Tong, He, Yang, Yang, and
  Zhu]{xiong2020nonseparable}
Shiying Xiong, Yunjin Tong, Xingzhe He, Shuqi Yang, Cheng Yang, and Bo~Zhu.
\newblock Nonseparable symplectic neural networks.
\newblock \emph{arXiv preprint arXiv:2010.12636}, 2020.

\bibitem[Zakwan et~al.(2023)Zakwan, d’Angelo, and
  Ferrari-Trecate]{zakwan2023universal}
Muhammad Zakwan, Massimiliano d’Angelo, and Giancarlo Ferrari-Trecate.
\newblock Universal approximation property of hamiltonian deep neural networks.
\newblock \emph{IEEE Control Systems Letters}, 7:\penalty0 2689--2694, 2023.

\bibitem[Zhu et~al.(2020)Zhu, Jin, and Tang]{zhu2020deep}
Aiqing Zhu, Pengzhan Jin, and Yifa Tang.
\newblock Deep hamiltonian networks based on symplectic integrators.
\newblock \emph{arXiv preprint arXiv:2004.13830}, 2020.

\end{thebibliography}

\end{document}